\documentclass{article}

\usepackage{microtype}
\usepackage{graphicx}
\usepackage{subcaption}
\usepackage{booktabs} 
\usepackage{hyperref}

\usepackage[accepted]{icml2025} %

\usepackage{amsmath}
\usepackage{amssymb}
\usepackage{mathtools}
\usepackage{amsthm}
\usepackage{bm}
\usepackage{amsfonts}
\usepackage{bbm}

\usepackage[capitalize,noabbrev]{cleveref}

\usepackage{float}
\usepackage{xcolor}
\usepackage{tikz}

\newcommand{\Var}{\mathrm{Var}}
\newcommand{\E}{\mathbb{E}}

\newcommand{\KL}{D}

\newcommand{\defeq}{\vcentcolon =}

\newcommand{\R}{\mathbb{R}}

\newcommand{\N}{\mathbb{N}}

\newcommand{\ie}{\emph{i.e.}}

\renewcommand{\eqref}[1]{Eq.~(\ref{#1})}

\newcommand{\appref}[1]{App.~\ref{#1}}

\usepackage{amsmath}

\DeclareMathOperator*{\argmin}{arg\,min}

\newcommand{\statespace}{\mathcal{S}}
\newcommand{\actionspace}{\mathcal{A}}
\newcommand{\tokenspace}{\mathcal{T}}
\newcommand{\KLQlam}{KLQ} %

\theoremstyle{plain}
\newtheorem{theorem}{Theorem}[section]
\newtheorem{proposition}[theorem]{Proposition}
\newtheorem{lemma}[theorem]{Lemma}
\newtheorem{corollary}[theorem]{Corollary}
\theoremstyle{definition}

\theoremstyle{remark}
\newtheorem{remark}[theorem]{Remark}

\usepackage[textsize=tiny]{todonotes}

\icmltitlerunning{A Token-level Action-Value perspective on Online RLHF}

\begin{document}

\twocolumn[
\icmltitle{KL-Regularised Q-Learning: A Token-level Action-Value perspective on Online RLHF}

\icmlsetsymbol{equal}{*}

\begin{icmlauthorlist}
\icmlauthor{Jason R. Brown}{equal,CL}
\icmlauthor{Lennie Wells}{equal,StatsLab}
\icmlauthor{Edward James Young}{equal,CBL}
\icmlauthor{Sergio Bacallado}{StatsLab}
\end{icmlauthorlist}

\icmlaffiliation{CBL}{Computational and Biological Learning Group, Department of Engineering, University of Cambridge, Cambridge, UK}
\icmlaffiliation{CL}{Department of Computer Science and Technology, University of Cambridge, Cambridge, UK}
\icmlaffiliation{StatsLab}{Statistics Laboratory, Department of Pure Mathematics and Mathematical Statistics, University of Cambridge, UK}

\icmlcorrespondingauthor{Edward James Young}{ey245@cam.ac.uk}
\icmlcorrespondingauthor{Lennie Wells}{ww347@cam.ac.uk}
\icmlcorrespondingauthor{Jason R. Brown}{jrb239@cam.ac.uk}
\icmlkeywords{ICML, Reinforcement Learning, Language Modelling, Reward Modelling}

\vskip 0.3in
]
\printAffiliationsAndNotice{\icmlEqualContribution} 

\begin{abstract}
Proximal Policy Optimisation (PPO) is an established and effective policy gradient algorithm used for Language Model Reinforcement Learning from Human Feedback (LM-RLHF). 
PPO performs well empirically but has a heuristic motivation and handles the KL-divergence constraint used in LM-RLHF in an ad-hoc manner.
In this paper, we develop a a new action-value RL method for the LM-RLHF setting, KL-regularised Q-Learning (KLQ).
We then show that our method is equivalent to a version of PPO in a certain specific sense, despite its very different motivation.
Finally, we benchmark KLQ on two key language generation tasks---summarisation and single-turn dialogue. We demonstrate that KLQ performs on-par with PPO at optimising the LM-RLHF objective, and achieves a consistently higher win-rate against PPO on LLM-as-a-judge evaluations.

\end{abstract}

\section{Introduction}

LLMs are typically tuned through a three-stage process: 
generative pre-training, supervised fine tuning (SFT), and Language Model-Reinforcement Learning from Human Feedback (LM-RLHF) \cite{ouyang_training_2022,bai_training_2022,huang_n_2024}.
These foundational works used the Proximal Policy Optimisation (PPO) algorithm for LM-RLHF, which is still widely used and viewed as the canonical algorithm for LM-RLHF \citep{xu_is_2024}.

A number of recent proposals have built upon PPO to develop alternative algorithms for the LM-RLHF setting, notably GRPO \citep{shao_deepseekmath_2024}, which works with the completion-level rather than token-level MDP and uses group-rollouts to compute advantage estimates.
There have also been a number of action-value methods proposed for LM-RLHF, but typically these work in the completion-level RL setting \citep{snell_offline_2023}, or are offline RL methods. We give a more thorough literature review in \cref{sec:related-work}.

By contrast, we sought to develop an \emph{online} \emph{action-value} method specifically tailored to the \emph{token-level} LM-RLHF setting.
We introduce our algorithm, KL-regularised Q-Learning (\KLQlam{}), in \cref{sec:KLQ}. 
\KLQlam{} exploits the discrete state-action space, and the presence of KL-regularisation against a base-policy.
We use $\lambda$-returns to construct regressions targets and leverage a powerful analytic correspondence between optimal policies and action-values in the KL-regularised setting.

Following our presentation of \KLQlam{}, we go on to show analytically that the updates implemented by \KLQlam{} are equivalent to updates performed by (a modified version of) PPO, despite the loss function for \KLQlam{} being substantially simpler.
This result was largely inspired by the seminal work of \citet{schulman_equivalence_2018}.

Finally, we compare the performance of \KLQlam{} to PPO on two standard benchmarks: TL;DR for summarisation, and Anthropic-HH for single-turn dialogue.
We demonstrate that KLQ optimises the LM-RLHF objective similarly well to PPO, with equal per-update compute cost. Moreover, we observe that the policy learned by KLQ outperforms the policy learned by PPO in direct LLM-as-a-judge evaluations.  

We hope that these results will help the community improve its understanding of what is important for LM-RLHF, and open the door to future action-value approaches.

\section{Background}\label{sec:background}

We will suppose that the reader is familiar with standard Reinforcement Learning (RL) \citep{sutton_reinforcement_2018} and with the three stage LLM pipeline of generative pre-training, supervised fine-tuning, and RLHF \citep{ouyang_training_2022}.
Instead, we focus on the less common framework of KL-regularised RL, which is a key motivation for our method.

\subsection{KL-regularised RL}

KL-regularised RL involves sequential interactions between an agent and an environment. At time $t$, the agent finds itself in state $s_t$ from a set $\statespace$ of possible states. The agent then samples an action $a_t$ from some set $\actionspace$ of possible actions according to a stochastic policy $\pi(\cdot|s)$. The environment then administers the agent some reward $r_{t+1} \in \R$ and transitions into a new state $s_{t+1}$.
Let $\gamma \in (0, 1)$ be a discount factor and $T$ be the (possibly infinite) time-step on which a terminal state is reached and the episode ends.
Let $\pi_b$ be a known (reliable) reference policy and $\tau > 0$ be a temperature parameter controlling the strength of the regularisation to this policy. 
Then the KL-regularised RL objective is to maximise the expectation of the discounted, KL-regularised return
\begin{equation}\label{eq:kl-augmented return}
    G_t = \sum_{k = t}^{T-1} \gamma^{k-t}\left( r_{k+1} - \tau \gamma \KL(\pi_\theta||\pi_b)\left(s_{k+1}\right) \right) ,
\end{equation}
where $\KL(\pi_\theta||\pi_b)(s)$ denotes the KL divergence between the action distributions at state $s$:
\begin{equation}\label{eq:kl-divergence}
    \KL(\pi_\theta||\pi_b)(s) \defeq \sum_a \pi_{\theta}(a|s) \log\left( \frac{\pi_{\theta}(a|s)}{\pi_b(a|s)} \right).
\end{equation}

This KL-regularised RL setting can be viewed as generalizing that of entropy-regularised RL \citep{levine_reinforcement_2018, schulman_equivalence_2018}, which is more common in the online setting and for which a set of highly practical effective algorithms have been developed \citep{haarnoja_reinforcement_2017, haarnoja_soft_2019}. The theory of KL-regularised setting is a straightforward extension of the entropy-regularised theory. For completeness, we repeat core standard results in \appref{app:theoretical results}.

\begin{remark}\label{rmk:gamma-is-one}
    In the case where $\gamma = 1$, the regularisation over states in \cref{eq:kl-augmented return} is equivalent to a KL regularisation over full trajectories $(s_1, \dots, s_T)$. In the LM-RLHF setting we will only ever take $\gamma = 1$, but we will present our algorithm for general $\gamma$ to clarify how it fits in to this standard discounting framework.
\end{remark}

\subsection{LLMs and RLHF}\label{sec:llms and rlhf}
We take the original RLHF pipeline for fine-tuning
LLMs laid out in \citet{ziegler_fine-tuning_2020} as our starting point. This trains LLMs for use in a \textit{prompt-completion} setting.
RLHF methods can be viewed either on the token-level or the completion level.

To fix notation, write $\tokenspace$ for space of all tokens and $\tokenspace^*=\cup_{n \in \N} \tokenspace^n$ for the space of finite sequences of tokens.
We can then denote a language model as $\pi(a|s)$, where $a \in \tokenspace, s \in \tokenspace^*$.
We represent a prompt by a sequence of tokens $x \in \tokenspace^*$.
The model will need to return some completion $y \in \tokenspace^*$, by auto-regressive sampling up to some fixed length, or until a distinguished `end-of-sequence' (EOS) token has been sampled.
We will write $\pi(y|x)$ for the probability mass function corresponding to this auto-regressive generation.

Write $\pi_b(a|s)$ for the policy generated by the supervised fine-tuning stage. Write $R_\phi(x, y)$ for the \textit{reward model} trained on human preferences to estimate the quality of a response $y$ for a given query $x$.
We will refer to the outputs $R_\phi(x,y)$ of the reward model as `preference scores' from now on to avoid ambiguity with the notion of reward in the standard RL setting.

A reinforcement learning (RL) algorithm is then used to fine-tune the SFT model to generate completions that receive high preference scores. It is typical to regularise using the KL-divergence between the generation policy and a reference policy, typically obtained via SFT policy, and which we will denote by $\pi_b$ for consistency with the KL-regularised RL setting. The modified objective is given by: 
\begin{equation}\label{eq:rlhf objective preference and penalty}
    \bar{R}_\phi(x,y) = R_\phi(x,y) - \tau \log\left( \frac{\pi_{\theta}(y|x)}{\pi_b(y|x)}\right).
\end{equation}
The KL-divergence penalty ensures that completions remain plausible under the SFT policy, avoiding mode-collapse and degenerate behaviour from reward model overoptimisation \citep{gao_scaling_2022}.

There are two distinct ways to view this learning problem from the state-action-reward framework of RL. Either:
\begin{enumerate}
    \item \textbf{Completion-level view:} Treat queries $x$ as states and completions $y$ as actions, leading to a \emph{contextual bandit} setting. In this case, the KL-divergence is over the entire completion $y$. This approach is taken by the RLOO algorithm \citep{ahmadian_back_2024}. 
    \item \textbf{Token-level view:} Treat partial completions as states, $s_t = (x, y_{1:t})$, and next tokens as actions, $a_t = y_{t+1}$. In this case, a reward equal to the weighted negative token-wise KL-divergence, $-\tau \KL(\pi_\theta||\pi_b)(s_{t+1})$, is administered at every time step. Since the reward model evaluates full completions, a final reward $R_\phi$ for the completion is given at the EOS token. This leads to a \emph{Markov Decision Process (MDP)} viewpoint. We take this approach, in common with PPO. 
\end{enumerate}
For $\gamma = 1$, these views are equivalent, since 
\begin{equation*}
\begin{split}
    \mathrm{KL}&(\pi(y|x)||\pi_b(y|x)) \\
    &= \E\left[ \sum_{t=0}^{T-1} \mathrm{KL}(\pi(a_t|s_t)||\pi_b(a_t|s_t)) \right] \, .
\end{split}
\end{equation*}
With either view, the full completion $y$ counts as an episode, and one can view the return of a trajectory $y$ as resulting in a total \emph{return} $\bar{R}_\phi(x,y)$. 
Moreover, following \cref{rmk:gamma-is-one}, this is equivalent to working in the KL-regularised framework with $\pi_b$ as the reference policy. In the remainder of our exposition we will write $\pi_b$ instead of $\pi_b$, since we are only interested in the LM-RLHF setting.

\paragraph{Proximal Policy Optimisation (PPO).}
The PPO algorithm trains a stochastic policy $\pi_{\theta}(a|s)$ and a value function $V_{\theta}(s)$, often implemented as two heads off a common neural network body. PPO is an on-policy algorithm which alternates between successive phases where experience is generated using the current policy $\pi_{\theta}(a|s)$ and then the policy and value function are updated.
The loss for the value function is based on a squared error between the value function output and value estimates formed from rollout experience.
In the typical version of PPO, PPO-clip, the policy is trained using a `clipped' objective, which increases the probability of high advantage actions, but removes the incentive for it to deviate too much from the previous policy iterate, $\pi_\text{old}$.
Letting $$\rho_\theta(a,s) = \frac{\pi_\theta(a|s)}{\pi_\text{old}(a|s)},$$ the objective takes the form
\begin{align}\label{eq:ppo-clip objective}
    \min \left(
        \rho_\theta(a,s)
        \hat A(s,a),
        \text{clip}^{1+\epsilon}_{1-\epsilon} \left[
            \rho_\theta(a,s)
            \right] \hat A(s,a)   
    \right),
\end{align}
where $\hat A(s,a)$ is the estimated advantage for taking action $a$ in state $s$ based on the current value function, and $\epsilon$ is the clipping ratio.
For concreteness, we give complete pseudocode for PPO-clip in the LM-RLHF setting in \cref{app:psuedocode}; we refer the reader to \citet{huang_n_2024} for a careful discussion of the algorithms implementation details.
Another variant of PPO that will be of interest is PPO-penalty.
This has a slightly different policy objective,
\begin{equation}\label{eq:ppo-penalty objective}\small
    \E_{\pi_{\rm old}}\left[
    \rho_\theta(a,s)
    \hat{A}(s,a) - \beta \KL(\pi_{\rm old} || \pi_\theta)(s) - \tau \KL(\pi_\theta || \pi_b)(s) \right],
\end{equation}
that incorporates a KL-penalty to the previous iterate instead of clipping.

\section{The \KLQlam{}  Algorithm}\label{sec:KLQ}

Our proposed algorithm, \KLQlam{}, is like PPO in that it is \emph{on-policy} and consists of the following two phases:
\begin{enumerate}
    \item Gather experiences from rollouts according to the current policy.
    \item Updating the parameters of language model neural networks using minibatch gradients to optimise a certain objective.
\end{enumerate}

However, whereas PPO uses a policy-gradient objective, KLQ uses the following $\ell^2$ loss between the \emph{action-value function} $Q_{\theta}(s, a)$ and value-estimates $\hat{G}$:
\begin{equation}\label{eq:ell-2 loss}
    \mathcal{L}(\theta) = \E\left[ \left( Q_{\theta}(s, a) - \hat{G} \right)^2 \right],
\end{equation}
where the expectation is over state, action, value-estimate triples.

There are two core components of our algorithm to specify: the construction of $\hat{G}$, and the parametrisation of $Q_\theta$.
We construct the value estimates $\hat{G}$ using $\lambda$-returns, as explained in section \cref{sec:lambda-returns}.
We construct $Q_\theta$ using a natural mapping between the policies, state-values and action-values in the KL-regularised setting, as explained in \cref{sec:action-value decomposition}. The full pseudocode for the \KLQlam{} algorithm can be found in \cref{app:psuedocode}.

\subsection{$\lambda$-return value estimator}\label{sec:lambda-returns}

A major consideration for a value-based methodology in the language modeling setting is the choice of value estimator. Many prominent action-value methods 
\citep{haarnoja_reinforcement_2017, mnih_human-level_2015} use one-step returns for their value estimate. However, one-step methods are inappropriate for the LM-RLHF setting, since the non-KL-regularisation portion of the reward signal (generated by the reward model, $R_\phi(x,y)$)  is very sparse, being administered only at the final time-step of the trajectory. We therefore choose a value formulation based on the $\lambda$-return framework. This allows for the effective propagation of reward information over longer time horizons.

The $\lambda$-return \citep{sutton_learning_1988} is a weighted combination of $n$-step returns. The parameter $\lambda$ acts as a \emph{truncation rate}. When $\lambda = 1$, we recover the zero-bias, high-variance full return. When $\lambda = 0$ we recover the high-bias, low-variance one-step-return. Intermediate values of the truncation rate $\lambda$ interpolate between these, allowing for an optimal bias-variance trade-off. 
PPO typically makes use of Generalised Advantage Estimation (GAE) \citep{schulman_high-dimensional_2018} to form advantage estimates, which uses $\lambda$-returns formulated in terms of the state-value function in order to estimate action-values. By contrast, \KLQlam{} applies the $\lambda$-return framework to action-value functions, and to prepare for our theoretical results in \cref{sec:equivalence}, we present these in the \emph{conservative} case. The conservative $\lambda$-returns $G^{\lambda,\alpha}_t$ corresponding to our action-value function $Q_\theta$ are defined via:
\begin{equation}\label{eq:lambda-return}
    G^{\lambda,\alpha}_t[Q_\theta] = \alpha \sum_{k=t}^{T-1} (\lambda \gamma)^{k-t} \delta_k + Q_{\theta}(s_t,a_t),
\end{equation}
where $\delta_t$ is the temporal-difference (TD) error for action-value functions in the KL-regularised setting,
\begin{align}\label{eq:TD-error}
    \delta_t &= r_{t+1} \nonumber \\
    &+ \gamma\left( \sum_a \pi_{\theta}(a|s_{t+1}) Q_{\theta}(s_{t+1},a) - \tau \KL(s_{t+1}) \right) \nonumber \\ 
    &- Q_{\theta}(s_t, a_t),
\end{align}
which quantifies the difference between the predicted value of actions and a one-step bootstrapped approximation of value. 
For further background on TD errors and their theory, see \citet{sutton_learning_1988, harutyunyan_q_2016}.

\begin{remark}\label{rmk:alpha-conservative-in-practice}
    The parameter $\alpha \in (0,1]$ scales the error term to bring the returns closer to the current action-value estimates. Thus, $\alpha \to 0$ is maximally conservative, while $\alpha \to 1$ recovers the standard $\lambda$-returns. We take $\alpha = 1$ in all our experiments, and only consider general $\alpha$ in our theoretical result. 
\end{remark}

Fitting the action-value function $Q_{\theta}(s, a)$ to the value-estimates $\hat{G}$ corresponds to a (partial) \emph{policy improvement} step - see Theorem \ref{thrm:lambda policy evaluation} in \cref{app:theoretical results}. In particular, it moves the action-value function $Q_{\theta}(s, a)$ towards the action-values of the policy $\pi$ used to generate the rollouts.

\subsection{Action-value decomposition}\label{sec:action-value decomposition}

One key barrier for the use of action-value function methods in the language modeling setting is the initialisation of the $Q$-function. For LM-RLHF, policy gradient methods have an obvious advantage, in that the policy can be initialised as the output of a previous stage in a training pipeline (see \cref{sec:llms and rlhf}). Without this initialisation, RL training for tasks in natural language would be completely infeasible. In order to use action-value methods for LM-RLHF, we develop a principled way to convert between a policy initialisation and an action-value initialisation.

\paragraph{General \texorpdfstring{$(\pi, V) \leftrightarrow Q$}{(policy, state-value) to action-value} mapping}\label{para:clean-mapping}
In the KL-regularised setting, policy improvement can be performed by greedily maximising the one-step expected value minus KL-divergence penalty, \ie, the new policy becomes the argmax of the following objective over $\pi$:
\begin{equation}
    \E_\pi\left[ Q(s,a) - \tau \log\left( \frac{\pi(a|s)}{\pi_b(a|s)} \right) \right].
\end{equation}
See Theorem \ref{thrm:kl-regularised policy improvement} and Corollary \ref{thrm:Boltzmann policy improvement} in \cref{app:theoretical results}.
Solving this optimisation problem yields the Boltzmann policy with respect to $Q$,
\begin{equation}\label{eq:boltzmann policy update}
    \pi[Q](a|s) = \pi_b(a|s) e^{( Q(s,a) - V[Q](s) )/\tau}.
\end{equation}
Here $V[Q](s)$ is the Boltzmann state-value function, which serves as a normalisation factor for the probability Boltzmann distribution,
\begin{equation}\label{eq:q to v mapping}
    V[Q](s) = \tau \log \left( \sum_a \pi_b(a|s)\exp( Q(s,a)/\tau )  \right).
\end{equation}
This relationship defines a clean mapping from action value functions, $Q$, to the corresponding (policy, state-value function) pairs \emph{at optimality}. %
This mapping can be inverted to recover an action-value function $Q$ from a (policy, state-value function) pair, $(\pi,V)$ via
\begin{equation}\label{eq:pi v to q mapping}
    Q[\pi,V](s,a) = \tau \log\left( \frac{\pi(a|s)}{\pi_b(a|s)} \right) + V(s) \,.
\end{equation}

\paragraph{Applying the mapping to our method} We leverage this mapping to define a convenient parametrisation of our action-value function, which enables us to move between $Q$-space and $(\pi,V)$-space seamlessly. In particular, we parametrise our action-value function using the policy, as:
\begin{equation}\label{eq:action-value decomposition}
    Q_{\theta}(s, a) = \tau \log\left( \frac{\pi_{\theta}(a|s)}{\pi_b(a|s)} \right) + V_{\theta}(s).
\end{equation}
The policy $\pi_{\theta}(a|s)$ is initialised as the SFT-policy, $\pi_b$, and the Boltzmann state-value function $V$ is a randomly initialised value head added to the policy network. 

Under this parametrization, computation of the TD-error \eqref{eq:TD-error} is particularly efficient. The TD-error now becomes:
\begin{equation}\label{eq:td-error with boltzmann state-value}
    \delta_t = r_{t+1} + \gamma V_{\theta}(s_{t+1}) 
    - Q_{\theta}(s_t, a_t).
\end{equation}
This follows from the fact that the log-probability term in the action-value function, \eqref{eq:action-value decomposition}, and the log-probabilities in the KL-divergence penalty, \eqref{eq:kl-divergence}, exactly cancel.

Note that the policy component of the action-value function $\pi_{\theta}(a|s)$ is never trained by itself on any kind of policy objective. Instead, the action-value function defined by \eqref{eq:action-value decomposition} is trained via the $\ell^2$ loss, \eqref{eq:ell-2 loss}, with $\lambda$-return value estimates, \eqref{eq:lambda-return}. Because the action-value function is parametrised in terms of the policy, \eqref{eq:action-value decomposition}, training the action-value function automatically updates the policy. Furthermore, our parametrisation ensures that the policy is always Boltzmann with respect to the action-value function, \eqref{eq:boltzmann policy update}. Accordingly, the policy improvement steps in \KLQlam{} happen \emph{implicitly}, as a result of our parameterisation \eqref{eq:action-value decomposition}.

\section{Equivalence between KLQ and PPO updates}\label{sec:equivalence}

In this section, we demonstrate that the updates performed by KLQ in $Q$-space are equivalent to the updates performed by a modified version of PPO in $(\pi,V)$-space.

We begin by considering the loss function used by the KLQ algorithm, \ref{eq:ell-2 loss}. Define a sequence $(Q_k)_k$ by iteratively minimising this loss function, starting at some $Q_0$:
\begin{equation}\label{eq:q-space update rule}
Q_{k+1} \gets \argmin_{Q} \E\left[ \left( Q(s,a) -  G^{\lambda, \alpha}[Q_k](s,a) \right)^2  \right]
\end{equation}
where we have made explicit the dependence of $G^{\lambda, \alpha}$ on the action-value function $Q_k$ and the starting state-action pair.

We now construct update rules in $(\pi,V)$-space which lead to the equivalent sequence of iterates; we state this formally in \cref{prop:equivalence}, and give a proof in \cref{app:full derivation}. 

We begin by defining advantage estimates via:
\begin{equation}
    \hat{A}[\pi,V](s,a) \defeq G^\lambda[\pi,V](s,a) - V(s) 
\end{equation}
We will take the sequence of policy iterates $\pi_{k+1}$ to be the unique maximisers of the following objective (\emph{cf.} \eqref{eq:ppo-penalty objective}):
\begin{equation}\label{eq:pi objective}
\begin{split}
    &\E_{\pi_k}\left[ \frac{\pi(a|s)}{\pi_k(a|s)} \hat{A}[\pi_k, V_k](s,a) \right] \\[5pt]
    &\qquad - \beta \KL(\pi || \pi_k ) - \tau \KL(\pi || \pi_b)
\end{split}
\end{equation}
where the KL-divergences are averaged over the state-visitation distribution of $\pi_k$. This policy objective can be seen as a modified version of PPO-penalty. Instead of the policy ratio clipping trick employed by PPO-Clip, this objective penalises the KL-divergence between the current policy and the previous iterate. However, it differs from standard PPO-penalty in that the KL-divergence is reverse. The previous iteration regularisation strength $\beta$ which pre-multiplies this term is defined in terms of $\alpha$ and $\tau$ by:
\begin{equation}\label{eq:beta definition}
    \beta = \tau \left( \frac{1 - \alpha}{\alpha} \right).
\end{equation}
We take next iterate state-value function $V_{k+1}$ to be the unique minimiser of the following loss%
:
\begin{equation}\label{eq:v loss}
    \E\left[ \left( V(s) - y(s,a)  \right)^2  \right],
\end{equation}
where
\begin{equation}\label{eq:v loss detail}
    y(s,a) =  G^{\lambda,\alpha}[V_k, \pi_k](s,a) - \tau \log\left( \frac{\pi_{k+1}(a|s)}{\pi_b(a|s)} \right).
\end{equation}
This differs from the standard state-value loss used for PPO in two ways. Firstly, it uses the $\alpha$ conservative $\lambda$-returns in place of the standard $\lambda$-return. Secondly, the KL-penalty term uses $\pi_{k+1}$ rather than $\pi_k$.

\begin{proposition}\label{prop:equivalence}
    The sequence $(Q_k)_k$ defined by \cref{eq:q-space update rule} corresponds via the mapping from \cref{para:clean-mapping} to the sequence $(\pi_k, V_k)_k$ defined above by optimising \cref{eq:pi objective,eq:v loss}.
\end{proposition}

\subsection{Discussion of theoretical result}

To see the significance of these result, we note that if provided states are sampled on-policy, the objective in \cref{eq:pi objective} is identical to the PPO-penalty objective from \cref{eq:ppo-penalty objective}, except for the reversal of the KL-divergence between the new policy and the previous policy. Since soft action-value methods with conservative backup operators implicitly maximise this objective, we can interpret them as performing a variant of proximal policy optimisation, with \cref{eq:beta definition} providing a mapping between the degree of conservativeness $\alpha$ and the previous iterate regularisation strength $\beta$. 

Additionally, this result allows us to reinterpret the roll of bootstrapping within proximal policy optimisation algorithms. While bootstrapping for action-value methods has its grounding in iterative policy evaluation procedures, bootstrapping for proximal policy optimisation algorithms serves solely as a variance reduction technique which introduces undesirable bias into the process. However, the equivalence relationship laid out above allows us to reinterpret bootstrapping to form advantage estimates as implicitly implementing iterative policy evaluation. 

\section{Experiments}\label{sec:experiments}

We test \KLQlam{} against our baseline of PPO on two standard datasets used in existing works.
The first is \textbf{TL;DR} \citep{syed_dataset_2018}, for \textit{summarisation}, where the model must generate summaries of a Reddit posts.
The second is Anthropic-\textbf{HH} \citep{bai_training_2022}, for \textit{single-turn dialogue}, where the model must provide a final response to a conversation between a user and an assistant that is helpful and harmless.

We initialise our policy with supervised finetuned (SFT) versions of Pythia-1B \citep{biderman_pythia_2023}, specific to each dataset.
A Pythia-1B based reward model is used for TL;DR, and a GPT2-large \citep{radford_language_2019} based reward model is used for HH.
We ran experiments using a customised fork of the TRL library \citep{werra_trl_2020}, and inherited default hyperparameter values from the TRL implementation of PPO.
Unless otherwise specified, each training run uses these default hyperparameters, and consists of 75,000 episodes, taking around 5 hours on 4 A100 GPUs.
\Cref{app:exp_details} contains further details of our training setup.
We note that while our experiments are preliminary in scale, they serve as an important proof-of-concept for our theoretical insights.

We first present representative training curves for \KLQlam{} and PPO in \cref{sec:exp_train_curves}, which show that KLQ is competitive out-of-the-box.

We then present results from more extensive experiments to investigate the effect of varying the KL-penalty coefficient, $\tau$.
Since \KLQlam{} incorporates KL-regularisation in a more sophisticated manner than PPO, we suspected there may be a different trade-off between reward and KL-divergence for different KL-penalty coefficients: we investigate the Pareto frontier of this trade-off in \cref{sec:exp_reward_kl_pareto}.
We also suspected that there could be differences in qualitative aspects of our final completions: 
We explore this effect in \cref{sec:llm_judge} with LLM-as-a-Judge win-rates between the two models across the KL-penalty coefficients.

Additional ablation studies to understand the effects of varying learning rate and $\lambda$ are presented in \cref{sec:ablation_learning_rate,sec:ablation_lambda} respectively.

\subsection{Training Curves}\label{sec:exp_train_curves}
\begin{figure}[h]
    \centering
    \begin{subfigure}[b]{0.45\textwidth}
        \includegraphics[width=\textwidth]{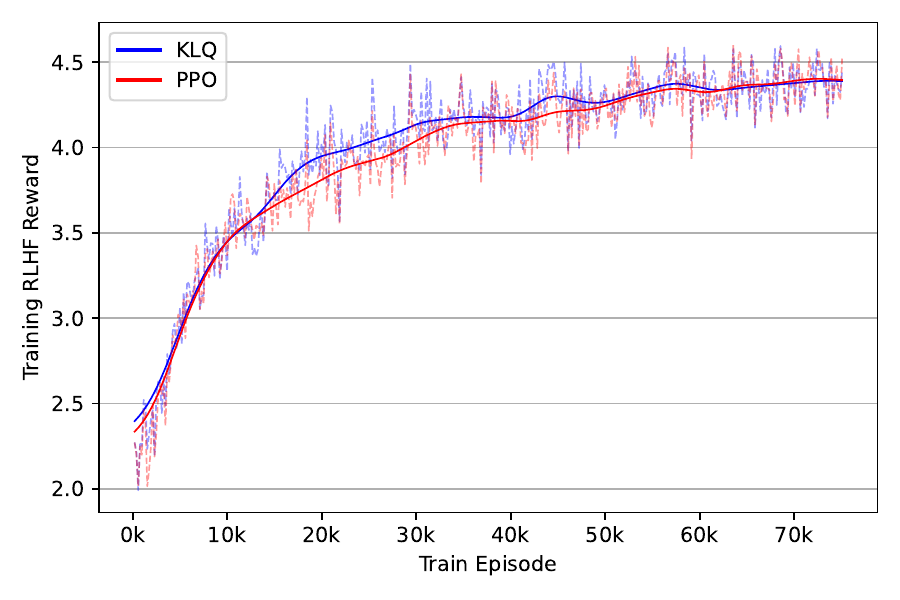}
        \caption{TL;DR}
    \end{subfigure}
    \begin{subfigure}[b]{0.45\textwidth}
        \includegraphics[width=\textwidth]{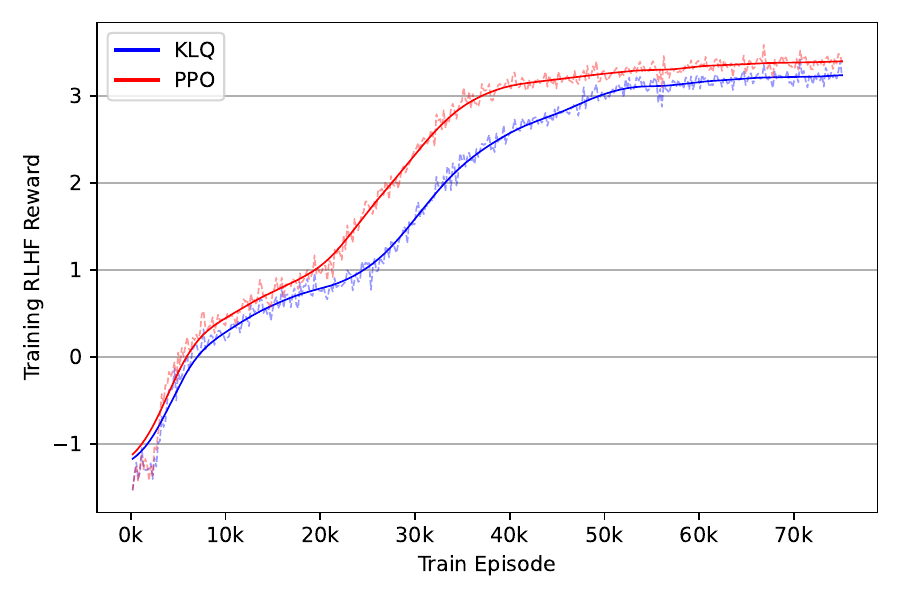}
        \caption{HH}
    \end{subfigure}
    \caption{RLHF rewards over training for \KLQlam{} and PPO on (a) TL;DR and (b) HH. We observe similar performance between the two algorithms, especially with respect to the final reward achieved.}
    \label{fig:train_curves}
\end{figure}

\Cref{fig:train_curves} plots RLHF rewards over training for \KLQlam{} and PPO on TL;DR and HH.
It shows \KLQlam{} and PPO achieving similar final results, with \KLQlam{} slightly lagging behind PPO at times on HH.
The validation curves show similar behaviour and are given in \cref{sec:exp_val_curves}, along with training wall-clock times, which are consistent for both algorithms.
This is particularly promising given that the default hyperparameters we used had previously been optimised for PPO, and we may expect significant improvements to KLQ from further hyperparameter tuning.

\subsection{Pareto Frontier of Reward and KL-Divergence}\label{sec:exp_reward_kl_pareto}

\begin{figure}[h]
    \centering
    \begin{subfigure}[b]{0.45\textwidth}
        \includegraphics[width=\textwidth]{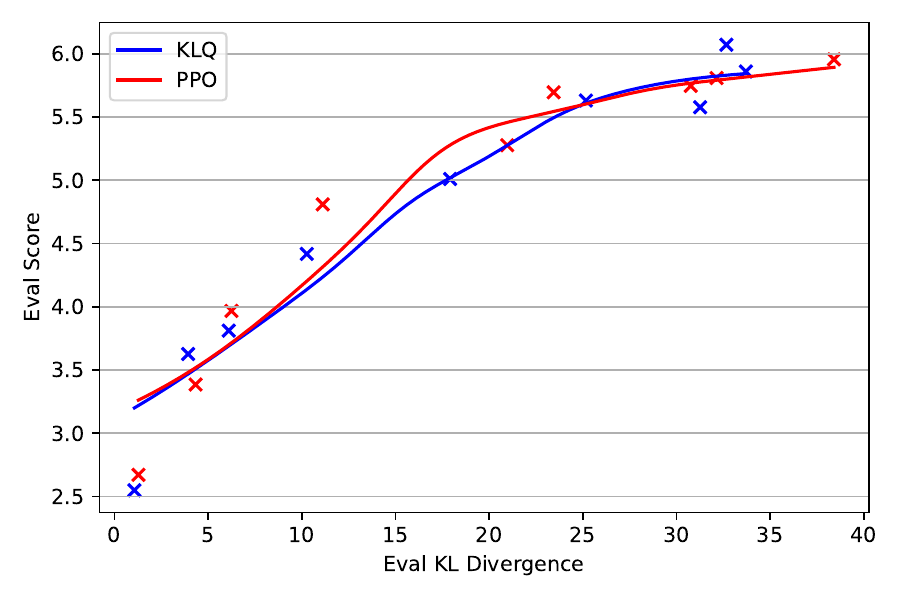}
        \caption{TL;DR}
    \end{subfigure}
    \begin{subfigure}[b]{0.45\textwidth}
        \includegraphics[width=\textwidth]{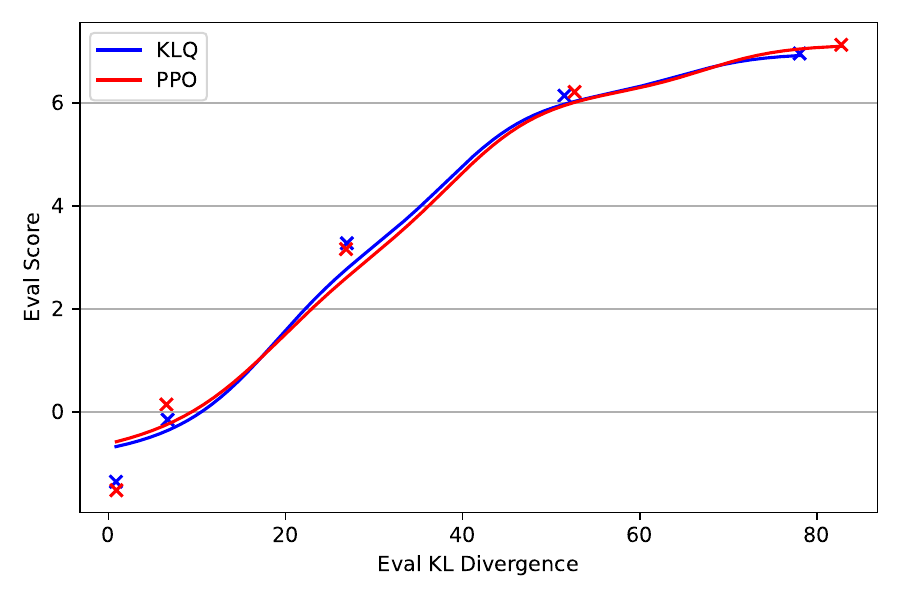}
        \caption{HH}
    \end{subfigure}
    \caption{Relationship between evaluation reward model score and KL-divergence to the SFT policy at the end of training. Curves are plotted using Gaussian smoothing. Whilst there are some differences, they mostly appear to be due to noise.}
    \label{fig:pareto}
\end{figure}

\Cref{fig:pareto} illustrates the relationship between reward model score and KL-divergence to the SFT policy at the end of training, when evaluated on a held-out validation set.
The KL-penalty coefficient, $\tau$, was varied to generate a diverse range of samples.
On both datasets we see very similar performance of the two algorithms.
The lack of monotonicity in the plot also suggests that there is significant statistical noise in the training process. This is particularly clear from the corresponding training curves, which we plot in \cref{sec:kl_val_curves}.

\subsection{LLM judges prefer KLQ}\label{sec:llm_judge}
\begin{figure}[t]
    \centering
    \includegraphics[width=0.45\textwidth]{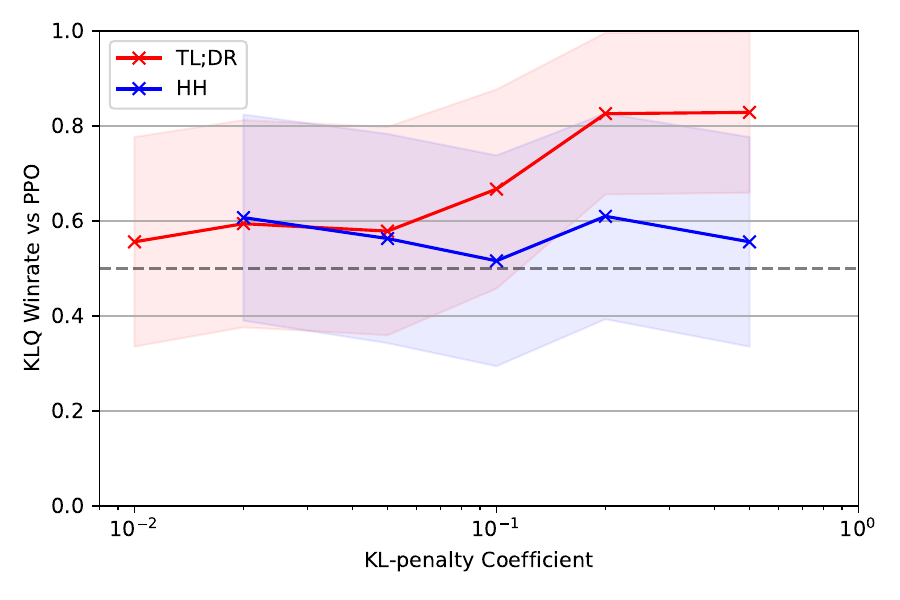}
    \caption{\KLQlam{}'s win-rate vs PPO on validation set prompt completions evaluated using LLM-as-a-Judge, across various KL-penalty coefficients. Confidence interval shown is Jeffrey's Interval with $\alpha=0.2$.}
    \label{fig:llm_judge}
\end{figure}

\Cref{fig:llm_judge} shows \KLQlam{}'s win-rate against PPO on validation set prompt completions evaluated using GPT-4o mini \citep{openai_gpt-4o_2024} as a judge.
We compare completions from the final models trained via each algorithm across a variety of different KL-penalty coefficients.
For each dataset, we use a fixed set of 32 validation prompts.
To enable Chain-of-Thought reasoning and aid interpretability, we prompt our judge to first compare the completions qualitatively before giving a preference;
\Cref{app:llm_judge} displays the full prompts we used to instruct our judge.
To control for ordering effects we present each prompt-completion pair to the judge in both orders (A vs B and B vs A); this gives a total of 64 queries per KL-penalty coefficient value per dataset.

Despite the similar performance of \KLQlam{} and PPO in \cref{sec:exp_train_curves,sec:exp_reward_kl_pareto}, here we see \KLQlam{} consistently outperforming PPO.
While the confidence intervals are wide due to the limited number of test cases, \KLQlam{} maintains a higher win rate across all tested KL-penalty coefficients on both datasets, with particularly strong performance at higher coefficients on TL;DR. These results suggest that \KLQlam{} may achieve better generalization than PPO, producing higher-quality outputs while matching PPO on quantitative metrics.
We hypothesise that this could be due to \KLQlam{}'s more theoretically grounded handling of the KL-divergence constraint.

\section{Related Work}\label{sec:related-work}

\subsection{Reinforcement learning (RL)}
\paragraph{Entropy-regularised RL.}
The KL-regularised setting specialises to the more familiar entropy regularised setting by setting the reference policy to be an (unnormalised) constant policy. 
The entropy regularised setting has its roots in inverse RL \citep{ziebart_maximum_2008}, but has recently become very popular for continuous control with the SQL \citep{haarnoja_reinforcement_2017} and SAC \citep{haarnoja_soft_2019} algorithms. Analogous versions of these algorithms have been developed for the discrete setting, namely Soft-DQN \citep{schulman_equivalence_2018,vieillard_munchausen_2020} and SAC-Discrete \citep{christodoulou_soft_2019} respectively. Of these, the closest to \KLQlam{} is Soft-DQN \citep{schulman_equivalence_2018, vieillard_munchausen_2020}. \KLQlam{} differentiates itself from Soft-DQN in the following ways: \KLQlam{} is on-policy, while Soft-DQN is off-policy; \KLQlam{} utilises a special decomposition of the action-value function, \eqref{eq:action-value decomposition}; and \KLQlam{} uses $\lambda$-returns rather than simple one-step returns.

\paragraph{Action-value function decomposition}\label{sec:action-value function decomposition related works}

The decomposition of the action-value function, \eqref{eq:action-value decomposition}, has appeared previously within the RL literature. Here we highlight two works which make use of it, and explain how we build on these. 

\citet{schulman_equivalence_2018} demonstrates that, when the action-value function is decomposed in this way, gradient steps on the action-value loss are equivalent to a mixture of policy gradient steps (for the KL-regularised setting) and gradient steps on the state-value loss. This was a core inspiration for \cref{sec:equivalence}. However, in order for their equivalence to hold, new on-policy rollouts must be sampled after every gradient step on the action-value function. Accordingly, \KLQlam{} is not reducible to a simple policy gradient method, since we perform multiple epochs of minibatch gradient descent on the action-value loss, \eqref{eq:ell-2 loss}, for each new set of on-policy rollouts, as opposed to a single step. Additionally, \citet{schulman_equivalence_2018} is primarily a theoretical contribution, and does not propose a novel algorithm which relies upon this equivalence. Our work aims to fill this gap with \KLQlam{}. 

\citet{zhu_value-based_2021} presents the VCWCV algorithm, which utilises a similar action-value decomposition to \KLQlam{}, with two key differences. First, VCWCV considers only the entropy regularised setting, taking the reference policy to be the (unnormalised) constant policy. Secondly, VCWCV is a continuous control algorithm, and parametrises the policy as a Beta distribution conditional on the state. This parametrisation places a significant restriction on the expressivity of the Q-function. By contrast, our decomposition in the discrete setting places no restrictions on the form of the Q-function. \KLQlam{} is further differentiated from VCWCV by its use of $\lambda$-returns and on-policy sampling.

\paragraph{$\lambda$-returns.} The \KLQlam{} algorithm uses the $\lambda$-return action-value estimator, \eqref{eq:lambda-return}, with the TD-error given by \eqref{eq:td-error with boltzmann state-value}. This estimator corresponds to the KL-regularised version of the on-policy $Q^\pi(\lambda)$ estimator \citep{harutyunyan_q_2016} for policy evaluation. In fact, because our policy $\pi$ is always Boltzmann with respect to the action-value function, our estimator also coincides with the KL-regularised version of the $Q^*(\lambda)$ estimator. \citet{harutyunyan_q_2016} demonstrated that, provided the sampling distribution $\mu$ is sufficiently close to the target policy $\pi$, the $Q^\pi(\lambda)$ estimator still defines a valid contraction mapping with fixed point given by the target action-value function. This raises the possibility of an off-policy formulation of \KLQlam{}, which reuses previously sampled trajectories. However, since we are typically not limited by rollout generation in the language modeling setting, we decided not to make this design choice in \KLQlam{}. 

\subsection{RLHF}

\paragraph{Group Rollouts.}
Various recent works have derived new algorithms for RLHF using grouped rollouts. By comparing preference scores given to different completions of the same prompt, it is possible to estimate advantages directly without the use of a value head. Some notable examples of this include: RLOO \citep{ahmadian_back_2024}, which apply the elegant REINFORCE Leave-One-Out procedure of \citet{kool_buy_2019}; GRPO \citep{shao_deepseekmath_2024}, which was used for DeepSeekMath; and CoPG \cite{flet-berliac_contrastive_2025}, a very flexible recent proposal that uses a squared loss to align differences in policy log probabilities to differences in preference score. Note however that all these methods work in the completion-level setting and do not have an explicit method for token-level attribution.

\paragraph{DPO.} \KLQlam{} leverages a conversion relationship between action-values and policies, \eqref{eq:action-value decomposition}, which allows us to train a policy using a value-based objective. Similarly, Direct Preference Optimisation (DPO) \citep{rafailov_direct_2024} leverages a conversion relationship between reward models $R_\phi$ and policies in the contextual bandit setting to train a policy using a reward modelling objective. The key practical advantage of DPO (shared by subsequent preference optimisation techniques) is that it allows the policy $\pi$ to be trained directly on the dataset of preferences, avoiding the need to use RL training.
In interesting follow-up work, \citet{rafailov_r_2024}, show that DPO can in fact be interpreted as a form of $Q$-learning in the \emph{token-level} MDP.
Our work can also be interpreted as $Q$-learning in the token-level setting, and provides a further illustration that `your language model is secretly an action-value function'.

\section{Discussion}

\subsection{Off-policy flexibility}
In common with PPO, \KLQlam{} is an \emph{on-policy} algorithm, training only using rollouts generated from the current policy. Although policy-gradient methods can be made off-policy, doing so often entails significant complications \citep{wang_sample_2017, espeholt_impala_2018} or a loss of theoretical justification \citep{lillicrap_continuous_2019}. In contrast, a core advantage of action-value methods is that they can retain their theoretical grounding regardless of whether data is sampled on- or off-policy \citep{haarnoja_reinforcement_2017, mnih_human-level_2015}.  In future work we hope to explore the use of off-policy sampling as an extension to \KLQlam{}. In particular, off-policy data can include expert demonstrations---obviating the need for an SFT stage altogether---or prioritised experience replay \citep{schaul_prioritized_2016}, allowing the reuse of trajectories that were particularly insightful.

\subsection{State-value head learning}
Recent algorithms that use group rollouts do not need to add a state-value head to their language model for LM-RLHF.
\citet{ahmadian_back_2024} argue that the key benefit of this technique is to eliminate the bias introduced from bootstrapping using value estimates.
Our initial $\lambda$-ablation experiments (see App. \ref{sec:ablation_lambda}) suggest a benefit to bootstrapping with a value head, at least for the action-value learning case with a single completion per prompt. 

The value head is an essential component of our \KLQlam{} method, and we suspect there may be benefits of explicit token-level attribution. In future work we hope to investigate ways to combine this with the benefits of group-rollouts.

\subsection{Advantage normalisation}
The TRL implementation of PPO uses advantage whitening, and claims that this is important.
It is not clear how to establish an analogue of this procedure for KLQ due to the combined $\pi, V$ objective.
There may be significant empirical improvements from establishing an appropriate whitening procedure for KLQ.

\subsection{Limitations}
The main limitation with this work at present is that our experiments are relatively preliminary compared to the most recent works in the LM-RLHF space.
Firstly, each training run took around 5 hours on 4 A100 GPUs; we did not have the computational resources to perform longer experiments, large hyper-parameter sweeps, or many repeats to reduce the effects of noise.
Secondly, we used existing SFT and reward models from the Huggingface Hub. Whilst we are thankful of the community and the open-access models, it was hard to evaluate or guarantee the quality and properties of these models.
A more robust evaluation pipeline would involve training our own SFT and reward models tailored to our specific needs.
We would also be excited to evaluate our methodology on multi-step reasoning tasks, which are currently an exciting application area for language model RL.
Finally, there are certain aspects of the equivalence in 

\section{Conclusion}

We propose \KLQlam{}, a novel algorithm for Language Model-Reinforcement Learning from Human Feedback (LM-RLHF).
We benchmark our method against Proximal Policy Optimisation (PPO), a canonical algorithm from previous works. %
\KLQlam{} has a cleaner theoretical motivation than PPO, performs similarly at optimising the LM-RLHF objective, and generates completions that are often preferred in LLM-as-a-judge evaluations. 

We provide an analytic argument to show that optimising the \KLQlam{} objective is equivalent to optimising a modified version of the PPO objective, in a certain specific sense. This demonstrates a heretofore unappreciated connection between proximal policy optimisation algorithms and action-value methods. 
\KLQlam{}'s clean motivation and close theoretical connection to PPO helps give an insight into the learning process of LM-RLHF, and points towards deeper reasons as to why PPO is effective.

\section*{Author Contributions}
JB, LW, and EY spent an equal number of hours on the project, and contributed to writing code, running empirics and writing the manuscript. EY initially conceived the project, the KLQ algorithm and developed all the theoretical results. SB provided advice and guidance throughout the project.

\section*{Acknowledgements}
EY and JB would like to thank their PhD supervisors Yashar Ahmadian and Robert Mullins for their guidance over the project, and providing access to compute resources. All authors would like to thank Usman Anwar for feedback on early iterations of the work, and to the Meridian office for providing a venue for in-person collaboration.

\section*{Impact Statement}

This goal of this paper is to develop more effective fine-tuning procedures for large language models that produce more preferable completions. This goal should lead to more aligned models, which is clearly in the public interest; however, this could in turn lead to more powerful models being deployed, increasing the corresponding societal impacts and risks.

\bibliography{main}

\begin{thebibliography}{36}
\providecommand{\natexlab}[1]{#1}
\providecommand{\url}[1]{\texttt{#1}}
\expandafter\ifx\csname urlstyle\endcsname\relax
  \providecommand{\doi}[1]{doi: #1}\else
  \providecommand{\doi}{doi: \begingroup \urlstyle{rm}\Url}\fi

\bibitem[Ahmadian et~al.(2024)Ahmadian, Cremer, Gallé, Fadaee, Kreutzer,
  Pietquin, Üstün, and Hooker]{ahmadian_back_2024}
Ahmadian, A., Cremer, C., Gallé, M., Fadaee, M., Kreutzer, J., Pietquin, O.,
  Üstün, A., and Hooker, S.
\newblock Back to {Basics}: {Revisiting} {REINFORCE} {Style} {Optimization} for
  {Learning} from {Human} {Feedback} in {LLMs}, February 2024.
\newblock URL \url{http://arxiv.org/abs/2402.14740}.
\newblock arXiv:2402.14740 [cs].

\bibitem[Bai et~al.(2022)Bai, Jones, Ndousse, Askell, Chen, DasSarma, Drain,
  Fort, Ganguli, Henighan, Joseph, Kadavath, Kernion, Conerly, El-Showk,
  Elhage, Hatfield-Dodds, Hernandez, Hume, Johnston, Kravec, Lovitt, Nanda,
  Olsson, Amodei, Brown, Clark, McCandlish, Olah, Mann, and
  Kaplan]{bai_training_2022}
Bai, Y., Jones, A., Ndousse, K., Askell, A., Chen, A., DasSarma, N., Drain, D.,
  Fort, S., Ganguli, D., Henighan, T., Joseph, N., Kadavath, S., Kernion, J.,
  Conerly, T., El-Showk, S., Elhage, N., Hatfield-Dodds, Z., Hernandez, D.,
  Hume, T., Johnston, S., Kravec, S., Lovitt, L., Nanda, N., Olsson, C.,
  Amodei, D., Brown, T., Clark, J., McCandlish, S., Olah, C., Mann, B., and
  Kaplan, J.
\newblock Training a {Helpful} and {Harmless} {Assistant} with {Reinforcement}
  {Learning} from {Human} {Feedback}, April 2022.
\newblock URL \url{http://arxiv.org/abs/2204.05862}.
\newblock arXiv:2204.05862 [cs].

\bibitem[Biderman et~al.(2023)Biderman, Schoelkopf, Anthony, Bradley,
  O’Brien, Hallahan, Khan, Purohit, Prashanth, Raff, and
  {others}]{biderman_pythia_2023}
Biderman, S., Schoelkopf, H., Anthony, Q.~G., Bradley, H., O’Brien, K.,
  Hallahan, E., Khan, M.~A., Purohit, S., Prashanth, U.~S., Raff, E., and
  {others}.
\newblock Pythia: {A} suite for analyzing large language models across training
  and scaling.
\newblock In \emph{International {Conference} on {Machine} {Learning}}, pp.\
  2397--2430. PMLR, 2023.

\bibitem[Christodoulou(2019)]{christodoulou_soft_2019}
Christodoulou, P.
\newblock Soft {Actor}-{Critic} for {Discrete} {Action} {Settings}, October
  2019.
\newblock URL \url{http://arxiv.org/abs/1910.07207}.
\newblock arXiv:1910.07207 [cs].

\bibitem[Espeholt et~al.(2018)Espeholt, Soyer, Munos, Simonyan, Mnih, Ward,
  Doron, Firoiu, Harley, Dunning, Legg, and Kavukcuoglu]{espeholt_impala_2018}
Espeholt, L., Soyer, H., Munos, R., Simonyan, K., Mnih, V., Ward, T., Doron,
  Y., Firoiu, V., Harley, T., Dunning, I., Legg, S., and Kavukcuoglu, K.
\newblock {IMPALA}: {Scalable} {Distributed} {Deep}-{RL} with {Importance}
  {Weighted} {Actor}-{Learner} {Architectures}, June 2018.
\newblock URL \url{http://arxiv.org/abs/1802.01561}.
\newblock arXiv:1802.01561 [cs].

\bibitem[Flet-Berliac et~al.(2025)Flet-Berliac, Grinsztajn, Strub, Wu, Choi,
  Cremer, Ahmadian, Chandak, Azar, Pietquin, and
  Geist]{flet-berliac_contrastive_2025}
Flet-Berliac, Y., Grinsztajn, N., Strub, F., Wu, B., Choi, E., Cremer, C.,
  Ahmadian, A., Chandak, Y., Azar, M.~G., Pietquin, O., and Geist, M.
\newblock Contrastive {Policy} {Gradient}: {Aligning} {LLMs} on sequence-level
  scores in a supervised-friendly fashion, January 2025.
\newblock URL \url{http://arxiv.org/abs/2406.19185}.
\newblock arXiv:2406.19185 [cs].

\bibitem[Gao et~al.(2022)Gao, Schulman, and Hilton]{gao_scaling_2022}
Gao, L., Schulman, J., and Hilton, J.
\newblock Scaling {Laws} for {Reward} {Model} {Overoptimization}, October 2022.
\newblock URL \url{http://arxiv.org/abs/2210.10760}.
\newblock arXiv:2210.10760 [cs].

\bibitem[Haarnoja et~al.(2017)Haarnoja, Tang, Abbeel, and
  Levine]{haarnoja_reinforcement_2017}
Haarnoja, T., Tang, H., Abbeel, P., and Levine, S.
\newblock Reinforcement {Learning} with {Deep} {Energy}-{Based} {Policies},
  July 2017.
\newblock URL \url{http://arxiv.org/abs/1702.08165}.
\newblock arXiv:1702.08165 [cs].

\bibitem[Haarnoja et~al.(2019)Haarnoja, Zhou, Hartikainen, Tucker, Ha, Tan,
  Kumar, Zhu, Gupta, Abbeel, and Levine]{haarnoja_soft_2019}
Haarnoja, T., Zhou, A., Hartikainen, K., Tucker, G., Ha, S., Tan, J., Kumar,
  V., Zhu, H., Gupta, A., Abbeel, P., and Levine, S.
\newblock Soft {Actor}-{Critic} {Algorithms} and {Applications}, January 2019.
\newblock URL \url{http://arxiv.org/abs/1812.05905}.
\newblock arXiv:1812.05905 [cs].

\bibitem[Harutyunyan et~al.(2016)Harutyunyan, Bellemare, Stepleton, and
  Munos]{harutyunyan_q_2016}
Harutyunyan, A., Bellemare, M.~G., Stepleton, T., and Munos, R.
\newblock Q($\lambda$) with {Off}-{Policy} {Corrections}, August 2016.
\newblock URL \url{http://arxiv.org/abs/1602.04951}.
\newblock arXiv:1602.04951 [cs].

\bibitem[Huang et~al.(2024)Huang, Noukhovitch, Hosseini, Rasul, Wang, and
  Tunstall]{huang_n_2024}
Huang, S., Noukhovitch, M., Hosseini, A., Rasul, K., Wang, W., and Tunstall, L.
\newblock The {N}+ {Implementation} {Details} of {RLHF} with {PPO}: {A} {Case}
  {Study} on {TL};{DR} {Summarization}, March 2024.
\newblock URL \url{http://arxiv.org/abs/2403.17031}.
\newblock arXiv:2403.17031.

\bibitem[Kool et~al.(2019)Kool, Hoof, and Welling]{kool_buy_2019}
Kool, W., Hoof, H.~v., and Welling, M.
\newblock Buy 4 {REINFORCE} {Samples}, {Get} a {Baseline} for {Free}!
\newblock April 2019.
\newblock URL \url{https://openreview.net/forum?id=r1lgTGL5DE}.

\bibitem[Levine(2018)]{levine_reinforcement_2018}
Levine, S.
\newblock Reinforcement {Learning} and {Control} as {Probabilistic}
  {Inference}: {Tutorial} and {Review}, May 2018.
\newblock URL \url{http://arxiv.org/abs/1805.00909}.
\newblock arXiv:1805.00909 [cs].

\bibitem[Lillicrap et~al.(2019)Lillicrap, Hunt, Pritzel, Heess, Erez, Tassa,
  Silver, and Wierstra]{lillicrap_continuous_2019}
Lillicrap, T.~P., Hunt, J.~J., Pritzel, A., Heess, N., Erez, T., Tassa, Y.,
  Silver, D., and Wierstra, D.
\newblock Continuous control with deep reinforcement learning, July 2019.
\newblock URL \url{http://arxiv.org/abs/1509.02971}.
\newblock arXiv:1509.02971 [cs].

\bibitem[Mnih et~al.(2015)Mnih, Kavukcuoglu, Silver, Rusu, Veness, Bellemare,
  Graves, Riedmiller, Fidjeland, Ostrovski, Petersen, Beattie, Sadik,
  Antonoglou, King, Kumaran, Wierstra, Legg, and
  Hassabis]{mnih_human-level_2015}
Mnih, V., Kavukcuoglu, K., Silver, D., Rusu, A.~A., Veness, J., Bellemare,
  M.~G., Graves, A., Riedmiller, M., Fidjeland, A.~K., Ostrovski, G., Petersen,
  S., Beattie, C., Sadik, A., Antonoglou, I., King, H., Kumaran, D., Wierstra,
  D., Legg, S., and Hassabis, D.
\newblock Human-level control through deep reinforcement learning.
\newblock \emph{Nature}, 518\penalty0 (7540):\penalty0 529--533, February 2015.
\newblock ISSN 1476-4687.
\newblock \doi{10.1038/nature14236}.
\newblock URL \url{https://www.nature.com/articles/nature14236}.
\newblock Publisher: Nature Publishing Group.

\bibitem[OpenAI(2024)]{openai_gpt-4o_2024}
OpenAI.
\newblock {GPT}-4o {System} {Card}, October 2024.
\newblock URL \url{http://arxiv.org/abs/2410.21276}.
\newblock arXiv:2410.21276 [cs].

\bibitem[Ouyang et~al.(2022)Ouyang, Wu, Jiang, Almeida, Wainwright, Mishkin,
  Zhang, Agarwal, Slama, Ray, Schulman, Hilton, Kelton, Miller, Simens, Askell,
  Welinder, Christiano, Leike, and Lowe]{ouyang_training_2022}
Ouyang, L., Wu, J., Jiang, X., Almeida, D., Wainwright, C.~L., Mishkin, P.,
  Zhang, C., Agarwal, S., Slama, K., Ray, A., Schulman, J., Hilton, J., Kelton,
  F., Miller, L., Simens, M., Askell, A., Welinder, P., Christiano, P., Leike,
  J., and Lowe, R.
\newblock Training language models to follow instructions with human feedback,
  March 2022.
\newblock URL \url{http://arxiv.org/abs/2203.02155}.
\newblock arXiv:2203.02155 [cs].

\bibitem[Radford et~al.(2019)Radford, Wu, Child, Luan, Amodei, and
  Sutskever]{radford_language_2019}
Radford, A., Wu, J., Child, R., Luan, D., Amodei, D., and Sutskever, I.
\newblock Language {Models} are {Unsupervised} {Multitask} {Learners}.
\newblock 2019.

\bibitem[Rafailov et~al.(2024{\natexlab{a}})Rafailov, Hejna, Park, and
  Finn]{rafailov_r_2024}
Rafailov, R., Hejna, J., Park, R., and Finn, C.
\newblock From \$r\$ to \${Q}{\textasciicircum}*\$: {Your} {Language} {Model}
  is {Secretly} a {Q}-{Function}, August 2024{\natexlab{a}}.
\newblock URL \url{http://arxiv.org/abs/2404.12358}.
\newblock arXiv:2404.12358 [cs].

\bibitem[Rafailov et~al.(2024{\natexlab{b}})Rafailov, Sharma, Mitchell, Ermon,
  Manning, and Finn]{rafailov_direct_2024}
Rafailov, R., Sharma, A., Mitchell, E., Ermon, S., Manning, C.~D., and Finn, C.
\newblock Direct {Preference} {Optimization}: {Your} {Language} {Model} is
  {Secretly} a {Reward} {Model}, July 2024{\natexlab{b}}.
\newblock URL \url{http://arxiv.org/abs/2305.18290}.
\newblock arXiv:2305.18290 [cs].

\bibitem[Schaul et~al.(2016)Schaul, Quan, Antonoglou, and
  Silver]{schaul_prioritized_2016}
Schaul, T., Quan, J., Antonoglou, I., and Silver, D.
\newblock Prioritized {Experience} {Replay}, February 2016.
\newblock URL \url{http://arxiv.org/abs/1511.05952}.
\newblock arXiv:1511.05952 [cs].

\bibitem[Schulman et~al.(2018{\natexlab{a}})Schulman, Chen, and
  Abbeel]{schulman_equivalence_2018}
Schulman, J., Chen, X., and Abbeel, P.
\newblock Equivalence {Between} {Policy} {Gradients} and {Soft} {Q}-{Learning},
  October 2018{\natexlab{a}}.
\newblock URL \url{http://arxiv.org/abs/1704.06440}.
\newblock arXiv:1704.06440 [cs].

\bibitem[Schulman et~al.(2018{\natexlab{b}})Schulman, Moritz, Levine, Jordan,
  and Abbeel]{schulman_high-dimensional_2018}
Schulman, J., Moritz, P., Levine, S., Jordan, M., and Abbeel, P.
\newblock High-{Dimensional} {Continuous} {Control} {Using} {Generalized}
  {Advantage} {Estimation}, October 2018{\natexlab{b}}.
\newblock URL \url{http://arxiv.org/abs/1506.02438}.
\newblock arXiv:1506.02438 [cs].

\bibitem[Shao et~al.(2024)Shao, Wang, Zhu, Xu, Song, Bi, Zhang, Zhang, Li, Wu,
  and Guo]{shao_deepseekmath_2024}
Shao, Z., Wang, P., Zhu, Q., Xu, R., Song, J., Bi, X., Zhang, H., Zhang, M.,
  Li, Y.~K., Wu, Y., and Guo, D.
\newblock {DeepSeekMath}: {Pushing} the {Limits} of {Mathematical} {Reasoning}
  in {Open} {Language} {Models}, April 2024.
\newblock URL \url{http://arxiv.org/abs/2402.03300}.
\newblock arXiv:2402.03300 [cs].

\bibitem[Snell et~al.(2023)Snell, Kostrikov, Su, Yang, and
  Levine]{snell_offline_2023}
Snell, C., Kostrikov, I., Su, Y., Yang, M., and Levine, S.
\newblock Offline {RL} for {Natural} {Language} {Generation} with {Implicit}
  {Language} {Q} {Learning}, May 2023.
\newblock URL \url{http://arxiv.org/abs/2206.11871}.
\newblock arXiv:2206.11871 [cs].

\bibitem[Sutton(1988)]{sutton_learning_1988}
Sutton, R.~S.
\newblock Learning to predict by the methods of temporal differences.
\newblock \emph{Machine Learning}, 3\penalty0 (1):\penalty0 9--44, August 1988.
\newblock ISSN 1573-0565.
\newblock \doi{10.1007/BF00115009}.
\newblock URL \url{https://doi.org/10.1007/BF00115009}.

\bibitem[Sutton \& Barto(2018)Sutton and Barto]{sutton_reinforcement_2018}
Sutton, R.~S. and Barto, A.~G.
\newblock \emph{Reinforcement learning: {An} introduction, 2nd ed}.
\newblock Reinforcement learning: {An} introduction, 2nd ed. The MIT Press,
  Cambridge, MA, US, 2018.
\newblock ISBN 978-0-262-03924-6.
\newblock Pages: xxii, 526.

\bibitem[Syed et~al.(2018)Syed, Voelske, Potthast, and
  Stein]{syed_dataset_2018}
Syed, S., Voelske, M., Potthast, M., and Stein, B.
\newblock Dataset for generating {TL};{DR}, February 2018.
\newblock URL \url{https://zenodo.org/records/1168855}.

\bibitem[Vieillard et~al.(2020)Vieillard, Pietquin, and
  Geist]{vieillard_munchausen_2020}
Vieillard, N., Pietquin, O., and Geist, M.
\newblock Munchausen {Reinforcement} {Learning}, November 2020.
\newblock URL \url{http://arxiv.org/abs/2007.14430}.
\newblock arXiv:2007.14430 [cs].

\bibitem[Wang et~al.(2017)Wang, Bapst, Heess, Mnih, Munos, Kavukcuoglu, and
  Freitas]{wang_sample_2017}
Wang, Z., Bapst, V., Heess, N., Mnih, V., Munos, R., Kavukcuoglu, K., and
  Freitas, N.~d.
\newblock Sample {Efficient} {Actor}-{Critic} with {Experience} {Replay}, July
  2017.
\newblock URL \url{http://arxiv.org/abs/1611.01224}.
\newblock arXiv:1611.01224 [cs].

\bibitem[Werra et~al.(2020)Werra, Belkada, Tunstall, Beeching, Thrush, Lambert,
  Huang, Rasul, and Gallouédec]{werra_trl_2020}
Werra, L.~v., Belkada, Y., Tunstall, L., Beeching, E., Thrush, T., Lambert, N.,
  Huang, S., Rasul, K., and Gallouédec, Q.
\newblock {TRL}: {Transformer} {Reinforcement} {Learning}, 2020.
\newblock URL \url{https://github.com/huggingface/trl}.
\newblock Publication Title: GitHub repository.

\bibitem[Wolf et~al.(2020)Wolf, Debut, Sanh, Chaumond, Delangue, Moi, Cistac,
  Rault, Louf, Funtowicz, Davison, Shleifer, Platen, Ma, Jernite, Plu, Xu,
  Scao, Gugger, Drame, Lhoest, and Rush]{wolf_huggingfaces_2020}
Wolf, T., Debut, L., Sanh, V., Chaumond, J., Delangue, C., Moi, A., Cistac, P.,
  Rault, T., Louf, R., Funtowicz, M., Davison, J., Shleifer, S., Platen, P.~v.,
  Ma, C., Jernite, Y., Plu, J., Xu, C., Scao, T.~L., Gugger, S., Drame, M.,
  Lhoest, Q., and Rush, A.~M.
\newblock {HuggingFace}'s {Transformers}: {State}-of-the-art {Natural}
  {Language} {Processing}, July 2020.
\newblock URL \url{http://arxiv.org/abs/1910.03771}.
\newblock arXiv:1910.03771 [cs].

\bibitem[Xu et~al.(2024)Xu, Fu, Gao, Ye, Liu, Mei, Wang, Yu, and
  Wu]{xu_is_2024}
Xu, S., Fu, W., Gao, J., Ye, W., Liu, W., Mei, Z., Wang, G., Yu, C., and Wu, Y.
\newblock Is {DPO} {Superior} to {PPO} for {LLM} {Alignment}? {A}
  {Comprehensive} {Study}, April 2024.
\newblock URL \url{https://arxiv.org/abs/2404.10719v2}.

\bibitem[Zhu et~al.(2021)Zhu, Zhang, and Pan]{zhu_value-based_2021}
Zhu, J., Zhang, H., and Pan, Z.
\newblock Value-{Based} {Continuous} {Control} {Without} {Concrete}
  {State}-{Action} {Value} {Function}.
\newblock In Tan, Y. and Shi, Y. (eds.), \emph{Advances in {Swarm}
  {Intelligence}}, pp.\  352--364, Cham, 2021. Springer International
  Publishing.
\newblock ISBN 978-3-030-78811-7.
\newblock \doi{10.1007/978-3-030-78811-7_34}.

\bibitem[Ziebart et~al.(2008)Ziebart, Maas, Bagnell, and
  Dey]{ziebart_maximum_2008}
Ziebart, B.~D., Maas, A., Bagnell, J.~A., and Dey, A.~K.
\newblock Maximum entropy inverse reinforcement learning.
\newblock In \emph{Proceedings of the 23rd national conference on {Artificial}
  intelligence - {Volume} 3}, {AAAI}'08, pp.\  1433--1438, Chicago, Illinois,
  July 2008. AAAI Press.
\newblock ISBN 978-1-57735-368-3.

\bibitem[Ziegler et~al.(2020)Ziegler, Stiennon, Wu, Brown, Radford, Amodei,
  Christiano, and Irving]{ziegler_fine-tuning_2020}
Ziegler, D.~M., Stiennon, N., Wu, J., Brown, T.~B., Radford, A., Amodei, D.,
  Christiano, P., and Irving, G.
\newblock Fine-{Tuning} {Language} {Models} from {Human} {Preferences}, January
  2020.
\newblock URL \url{http://arxiv.org/abs/1909.08593}.
\newblock arXiv:1909.08593 [cs].

\end{thebibliography}
\bibliographystyle{main.bst}

\newpage
\appendix
\onecolumn

\section{Theoretical Background}\label{app:theoretical results}

In this appendix we recount core theoretical results for the KL-regularised setting and $\lambda$-returns. These results follow straightforwardly from the analogous results in the entropy-regularised setting, so do not constitute a major contribution of our work. We include them here only for clarity of presentation and completeness of exposition. All results in this section apply to a finite MDP, \ie, $|\mathcal{S} \times \mathcal{A}| < \infty$ with a discount factor $\gamma \in [0,1)$ and rewards bounded above by $R_{\rm max}$. We always assume that $\lambda \in [0,1]$. 

The \emph{soft action-value function} for a policy $\pi$, $Q^\pi(s,a)$ is defined by:
\begin{equation}\label{eq:soft action-value function}
    Q^\pi(s,a) \defeq \E_{\pi}\left[ G_t \middle| s_t = s, a_t = a \right] =  \E_{\pi}\left[ \sum_{k = t}^{T-1} \gamma^{k-t}\left( r_{k+1} - \tau \gamma D_\pi\left(s_{k+1}\right) \right) \middle| s_t = s, a_t = a \right] 
\end{equation}
where $G_t$ is the KL-augmented return, \eqref{eq:kl-augmented return}, and $D_\pi(s)$ is the KL-divergence between $\pi$ and $\pi_{\rm SFT}$ at $s$. We next present policy evaluation and improvement theorems for this setting.

\begin{lemma}[Bellman recursion for the soft action-value function]\label{thrm:bellman recursion for soft action-values}
    Let $B^\pi$ be the soft Bellman operator for $\pi$, which acts on action-value functions via:
    \begin{equation}\label{eq:soft bellman operator}
        \left[ B^\pi Q \right](s,a) = \E_\pi\left[ r_{t+1} + \gamma\left( Q(s_{t+1},a_{t+1}) - \tau D_\pi \left( s_{t+1} \right) \right) \middle| s = s_t,a = a_t  \right].
    \end{equation}
    Then the soft action-value function, $Q^\pi(s,a)$, satisfies the Bellman recursion relationship, $B^\pi Q^\pi = Q^\pi$.
\end{lemma}
\begin{proof}[Proof of Lemma \ref{thrm:bellman recursion for soft action-values}]
We begin with the definition of the soft action-value function:
\begin{align}
    Q^\pi(s,a) &= \E_{\pi}\left[ G_t \middle| s_t = s, a_t = a \right] \\
    &= \E_{\pi}\left[ \sum_{k = t}^{T-1} \gamma^{k-t}\left( r_{k+1} - \tau \gamma D_\pi\left(s_{k+1}\right) \right) \middle| s_t = s, a_t = a \right]
\end{align}

Splitting the sum to separate the immediate reward and future returns:
\begin{align}
    Q^\pi(s,a) &= \E_{\pi}\left[ r_{t+1} - \tau \gamma D_\pi(s_{t+1}) + \gamma\sum_{k = t+1}^{T-1} \gamma^{k-(t+1)}\left( r_{k+1} - \tau \gamma D_\pi\left(s_{k+1}\right) \right) \middle| s_t = s, a_t = a \right]
\end{align}

The second term inside the expectation is precisely the KL-augmented return from time $t+1$, denoted as $G_{t+1}$. By the definition of the soft action-value function, $\E_{\pi}\left[ G_{t+1} | s_{t+1}, a_{t+1} \right] = Q^\pi(s_{t+1},a_{t+1})$. Therefore:
\begin{align}
    Q^\pi(s,a) &= \E_{\pi}\left[ r_{t+1} - \tau \gamma D_\pi\left( s_{t+1} \right) + \gamma Q^\pi(s_{t+1},a_{t+1}) \middle| s_t = s, a_t = a \right] \\
    &= \E_{\pi}\left[ r_{t+1} + \gamma \left(  Q^\pi(s_{t+1},a_{t+1}) - \tau D_\pi\left( s_{t+1} \right) \right) \middle| s_t = s, a_t = a \right] \\
    &= \left[ B^\pi Q^\pi \right](s,a)
\end{align}

This confirms that $Q^\pi$ satisfies the Bellman recursion relationship, $B^\pi Q^\pi = Q^\pi$.
\end{proof}

\begin{theorem}[KL-regularised policy evaluation]\label{thrm:kl-regularised policy evaluation}
    The soft Bellman operator is a contraction mapping in the $\ell^\infty$ norm with contraction modulus $\gamma$.
    Accordingly, the soft action-value function $Q^\pi(s,a)$ is the unique fixed point of $B^\pi$, and any sequence of iterates converges to $Q^\pi(s,a)$ in the $\ell^\infty$ norm. 
\end{theorem}
\begin{proof}[Proof of Theorem \ref{thrm:kl-regularised policy evaluation}]
To prove that $B^\pi$ is a contraction mapping in the $\ell^\infty$ norm with contraction modulus $\gamma$, we need to show that for any two action-value functions $Q_1$ and $Q_2$:
\begin{equation}
    \|B^\pi Q_1 - B^\pi Q_2\|_{\infty} \leq \gamma \|Q_1 - Q_2\|_{\infty}
\end{equation}

Let us expand the difference for any state-action pair $(s,a)$:
\begin{align}
    &~ [B^\pi Q_1](s,a) - [B^\pi Q_2](s,a) \nonumber \\
    &= \E_\pi\left[ r_{t+1} + \gamma\left( Q_1(s_{t+1},a_{t+1}) - \tau D_\pi\left( s_{t+1} \right) \right) \middle| s = s_t,a = a_t  \right] \nonumber \\ 
    & - \E_\pi\left[ r_{t+1} + \gamma\left( Q_2(s_{t+1},a_{t+1}) - \tau D_\pi\left( s_{t+1} \right) \right) \middle| s = s_t,a = a_t  \right]
\end{align}

The reward terms $r_{t+1}$ and KL divergence terms $D_\pi(s_{t+1})$ cancel out, leaving:
\begin{align}
    |[B^\pi Q_1](s,a) - [B^\pi Q_2](s,a)| &= \gamma \left|\E_\pi\left[ Q_1(s_{t+1},a_{t+1}) - Q_2(s_{t+1},a_{t+1}) \middle| s = s_t,a = a_t  \right] \right|
\end{align}

By Jensen's inequality, $|\E[X]| \leq \E[|X|]$, we have:
\begin{align}
    |[B^\pi Q_1](s,a) - [B^\pi Q_2](s,a)| &\leq \gamma \E_\pi\left[ |Q_1(s_{t+1},a_{t+1}) - Q_2(s_{t+1},a_{t+1})| \middle| s = s_t,a = a_t  \right]
\end{align}

Since $\|Q_1 - Q_2\|_{\infty} = \max_{s,a} |Q_1(s,a) - Q_2(s,a)|$, we know that for any state-action pair $(s_{t+1},a_{t+1})$:
\begin{equation}
    |Q_1(s_{t+1},a_{t+1}) - Q_2(s_{t+1},a_{t+1})| \leq \|Q_1 - Q_2\|_{\infty}
\end{equation}

This gives us:
\begin{align}
    |[B^\pi Q_1](s,a) - [B^\pi Q_2](s,a)| &\leq \gamma \E_\pi\left[ \|Q_1 - Q_2\|_{\infty} \middle| s = s_t,a = a_t  \right] \\
    &= \gamma \|Q_1 - Q_2\|_{\infty}
\end{align}

Since this inequality holds for all state-action pairs $(s,a)$, we have:
\begin{equation}
    \|B^\pi Q_1 - B^\pi Q_2\|_{\infty} \leq \gamma \|Q_1 - Q_2\|_{\infty}
\end{equation}

This proves that $B^\pi$ is a contraction mapping with contraction modulus $\gamma < 1$. By the Banach fixed-point theorem, $B^\pi$ has a unique fixed point, and any sequence of iterates $Q_{n+1} = B^\pi Q_n$ converges to this fixed point in the $\ell^\infty$ norm.

We've already shown in Lemma \ref{thrm:bellman recursion for soft action-values} that $Q^\pi$ is a fixed point of $B^\pi$. Therefore, $Q^\pi$ is the unique fixed point of $B^\pi$, and any sequence of iterates converges to $Q^\pi$.
\end{proof}

The $\lambda$-returns used to formulate the value estimates for our action-value function loss $\mathcal{L}(\theta)$ are related to the $\lambda$ soft Bellman operator, defined by:
\begin{equation}\label{eq:lambda-bellman operator}
    B^\pi_\lambda = (1 - \lambda) \sum_{n \geq 1} \lambda^{n-1} \left[ B^\pi \right]^n.
\end{equation}
In particular, 
\begin{equation}
    \left[ B^\pi_\lambda Q \right](s,a) = \E_\pi\left[ G^\lambda_t | s_t = s, a_t = a \right]
\end{equation}
Note that the $\lambda$ soft Bellman operator is a weighted average of $n$-step Bellman operators, $\left[ B^\pi \right]^n$, with weights $(1 - \lambda)\lambda^{n-1}$.

\begin{theorem}[$\lambda$ policy evaluation]\label{thrm:lambda policy evaluation} 
    The $\lambda$ soft Bellman operator, $B^\pi_\lambda$, is a contraction mapping in the $\ell^\infty$ norm with contraction modulus \begin{equation*}
        \gamma \left(  \frac{1 - \lambda}{1 - \lambda \gamma} \right).
    \end{equation*}
    Accordingly, the soft action-value function $Q^\pi(s,a)$ is the unique fixed point of $B^\pi_\lambda$, and any sequence of iterates converges to $Q^\pi(s,a)$ in the $\ell^\infty$ norm. 
\end{theorem}
\begin{proof}[Proof of Theorem \ref{thrm:lambda policy evaluation}] 
For any two action-value functions $Q_1$ and $Q_2$, we have:
\begin{align}
    \|B^\pi_\lambda Q_1 - B^\pi_\lambda Q_2\|_{\infty} &= \left\|(1 - \lambda) \sum_{n \geq 1} \lambda^{n-1} \left( \left[ B^\pi \right]^n Q_1 - \left[ B^\pi \right]^n Q_2 \right) \right\|_{\infty}
\end{align}

Using the triangle inequality:
\begin{align}
    \|B^\pi_\lambda Q_1 - B^\pi_\lambda Q_2\|_{\infty} &\leq (1 - \lambda) \sum_{n \geq 1} \lambda^{n-1} \left\| \left[ B^\pi \right]^n Q_1 - \left[ B^\pi \right]^n Q_2 \right\|_{\infty}
\end{align}

From Theorem \ref{thrm:kl-regularised policy evaluation}, we know that $B^\pi$ is a contraction mapping with modulus $\gamma$. By the properties of contraction mappings, iterating $n$ times gives:
\begin{equation}
    \left\| \left[ B^\pi \right]^n Q_1 - \left[ B^\pi \right]^n Q_2 \right\|_{\infty} \leq \gamma^n \|Q_1 - Q_2\|_{\infty}
\end{equation}

Substituting this into our inequality:
\begin{align}
    \|B^\pi_\lambda Q_1 - B^\pi_\lambda Q_2\|_{\infty} &\leq (1 - \lambda) \sum_{n \geq 1} \lambda^{n-1} \gamma^n \|Q_1 - Q_2\|_{\infty} \\
    &= (1 - \lambda) \gamma \|Q_1 - Q_2\|_{\infty} \sum_{n \geq 1} \left( \lambda \gamma \right)^{n-1} \\
    &= (1 - \lambda) \gamma \|Q_1 - Q_2\|_{\infty} \frac{1}{1 - \lambda \gamma} \\
    &= \gamma \frac{1 - \lambda}{1 - \lambda \gamma} \|Q_1 - Q_2\|_{\infty}
\end{align}
This proves that $B^\pi_\lambda$ is a contraction mapping with the stated contraction modulus.

Now we need to show that $Q^\pi$ is the unique fixed point of $B^\pi_\lambda$. We know from Lemma \ref{thrm:bellman recursion for soft action-values} that $B^\pi Q^\pi = Q^\pi$. Substituting into the definition of $B^\pi_\lambda$:
\begin{align}
    B^\pi_\lambda Q^\pi &= (1 - \lambda) \sum_{n \geq 1} \lambda^{n-1} \left[ B^\pi \right]^n Q^\pi \\
    &= (1 - \lambda) \sum_{n \geq 1} \lambda^{n-1} Q^\pi \\
    &= (1 - \lambda) Q^\pi \sum_{n \geq 1} \lambda^{n-1} \\
    &= (1 - \lambda) Q^\pi \frac{1}{1 - \lambda} \\
    &= Q^\pi
\end{align}

Therefore, $Q^\pi$ is a fixed point of $B^\pi_\lambda$. Since $B^\pi_\lambda$ is a contraction mapping, this fixed point is unique by the Banach fixed-point theorem, and any sequence of iterates converges to $Q^\pi$ in the $\ell^\infty$ norm.
\end{proof}

Having established a theoretical basis for policy evaluation by performing regression on $\lambda$-returns, we now turn to policy improvement. In the classical setting, the greedy policy can be found by setting 
\begin{equation}
    \pi(a|s) \gets \arg\max_{\pi}  \E_\pi\left[ Q(s,a) \right]. 
\end{equation}

In the KL-regularised setting, this is generalised to:
\begin{equation}\label{eq:boltzmann policy}
    \pi(a|s) \gets \arg\max_{\pi} \left\{ \E_\pi\left[ Q(s,a) - \tau \log\left( \frac{\pi(a|s)}{\pi_{\rm SFT}(a|s)} \right) \right] \right\}
\end{equation}
\begin{theorem}[Greedy KL-regularised policy]\label{thrm:greedy kl-regularised policy} The policy which solves the optimisation problem in \Cref{eq:boltzmann policy} is given by the Boltzmann policy with respect to the current action-value function $Q$, \ie, 
\begin{equation}\label{eq:boltzmann policy definition}
    \pi_B(a|s) \defeq \pi_{\rm SFT}(a|s) e^{( Q(s,a) - V_B(s) )/\tau},
\end{equation}
where $V_B(s)$ is the Boltzmann state-value function, \begin{equation}\label{eq:boltzmann state-value function definition}
    V_B(s) = \tau \log \left( \E\left[ e^{Q(s,a)/\tau} \middle| a \sim \pi_{\rm SFT}(a|s) \right] \right).
\end{equation}
\end{theorem}

\begin{proof}
We shall solve the constrained optimisation problem given in \Cref{eq:boltzmann policy}. For a fixed state $s$, we can rewrite the objective function as:
\begin{equation}
    J(\pi) = \sum_a \pi(a|s) Q(s,a) - \tau \sum_a \pi(a|s) \log\left( \frac{\pi(a|s)}{\pi_{\rm SFT}(a|s)} \right)
\end{equation}

The optimisation is subject to the constraint that $\pi(\cdot|s)$ is a valid probability distribution:
\begin{equation}
    \sum_a \pi(a|s) = 1 \quad \text{and} \quad \pi(a|s) \geq 0 \quad \forall a \in \mathcal{A}
\end{equation}

Using the method of Lagrange multipliers, we introduce the Lagrangian:
\begin{equation}
    L(\pi, \lambda) = \sum_a \pi(a|s) Q(s,a) - \tau \sum_a \pi(a|s) \log\left( \frac{\pi(a|s)}{\pi_{\rm SFT}(a|s)} \right) - \lambda \left(\sum_a \pi(a|s) - 1 \right)
\end{equation}

Taking the partial derivative with respect to $\pi(a|s)$ for each action $a$:
\begin{align}
    \frac{\partial L}{\partial \pi(a|s)} &= Q(s,a) - \tau \log\left( \frac{\pi(a|s)}{\pi_{\rm SFT}(a|s)} \right) - \tau - \lambda
\end{align}

Setting this derivative to zero to find the critical points:
\begin{align}
    Q(s,a) - \tau \log\left( \frac{\pi(a|s)}{\pi_{\rm SFT}(a|s)} \right) - \tau - \lambda &= 0\\
    \tau \log\left( \frac{\pi(a|s)}{\pi_{\rm SFT}(a|s)} \right) &= Q(s,a) - \tau - \lambda
\end{align}

Taking the exponential of both sides:
\begin{align}
    \frac{\pi(a|s)}{\pi_{\rm SFT}(a|s)} &= e^{(Q(s,a) - \tau - \lambda)/\tau}\\
    \pi(a|s) &= \pi_{\rm SFT}(a|s) e^{(Q(s,a) - \tau - \lambda)/\tau}
\end{align}

To determine the value of $\lambda$, we use the constraint that $\sum_a \pi(a|s) = 1$:
\begin{align}
    \sum_a \pi(a|s) &= \sum_a \pi_{\rm SFT}(a|s) e^{(Q(s,a) - \tau - \lambda)/\tau} = 1\\
    e^{-(\tau + \lambda)/\tau} \sum_a \pi_{\rm SFT}(a|s) e^{Q(s,a)/\tau} &= 1
\end{align}

Solving for $e^{(\tau + \lambda)/\tau}$:
\begin{align}
    e^{(\tau + \lambda)/\tau} &= \sum_a \pi_{\rm SFT}(a|s) e^{Q(s,a)/\tau} \\
    (\tau + \lambda) &= \tau \log\left(  \sum_a \pi_{\rm SFT}(a|s) e^{Q(s,a)/\tau} \right) \\
    &= V_B(s) 
\end{align}
Substituting this back into our expression for $\pi(a|s)$:
\begin{align}
\pi(a|s) &= \pi_{\rm SFT}(a|s) e^{(Q(s,a) - \tau - \lambda)/\tau}\\
&= \pi_{\rm SFT}(a|s) e^{(Q(s,a) - V_B(s))/\tau}
\end{align}
This is precisely the Boltzmann policy defined in \Cref{eq:boltzmann policy definition}, as claimed.
\end{proof}

At each cycle of our algorithm, we perform a policy evaluation step by training on the $\lambda$-returns. Because of the parametrisation of our action-value function, we implicitly and automatically perform a full policy improvement step, \eqref{eq:boltzmann policy}. However, here we state a more general result regarding policy improvement which includes partial improvement steps.   
\begin{theorem}[KL-regularised policy improvement]\label{thrm:kl-regularised policy improvement}
     Then for any policy $\pi$, let $\pi_B$ be the corresponding Boltzmann policy, given by \eqref{eq:boltzmann policy definition}. If $\pi_{\rm new}$ satisfies 
     \begin{equation}\label{eq:improvement condition}
          D \left( \pi_{\rm new}(\cdot|s') || \pi_B(\cdot|s') \right) \leq D \left( \pi(\cdot|s') || \pi_B(\cdot|s') \right),~\forall s' \in \mathcal{S},
     \end{equation} 
     then $Q^{\pi_{\rm new}} \geq Q^\pi$. Moreover, for any state-action pair $(s,a)$ which has a non-zero probability of transitioning into a state $s'$ at which the inequality in \eqref{eq:improvement condition} is strict, we have that $Q^{\pi_{\rm new}}(s,a) > Q^\pi(s,a)$. 
\end{theorem}
\begin{proof}[Proof of Theorem \ref{thrm:kl-regularised policy improvement}]
We begin by expressing the next-state value in terms of the Boltzmann state value function and the KL divergence between a policy and the Boltzmann policy.

For any state $s'$ and policy $\pi$, consider the expected value:
\begin{align}
\E_{a' \sim \pi(\cdot|s')}\left[Q(s',a') - \tau \log\left(\frac{\pi(a'|s')}{\pi_{\rm SFT}(a'|s')}\right)\right]
\end{align}

Recall that the Boltzmann policy satisfies:
\begin{align}
\pi_B(a'|s') &= \pi_{\rm SFT}(a'|s') e^{(Q(s',a') - V_B(s'))/\tau}
\end{align}

Taking logarithms and rearranging:
\begin{align}
Q(s',a') &= \tau \log\left(\frac{\pi_B(a'|s')}{\pi_{\rm SFT}(a'|s')}\right) + V_B(s')
\end{align}

Substituting this into our expected value expression:
\begin{align}
    &\E_{a' \sim \pi(\cdot|s')}\left[Q(s',a') - \tau \log\left(\frac{\pi(a'|s')}{\pi_{\rm SFT}(a'|s')}\right)\right] \\
    &= \E_{a' \sim \pi(\cdot|s')}\left[\tau \log\left(\frac{\pi_B(a'|s')}{\pi_{\rm SFT}(a'|s')}\right) + V_B(s') - \tau \log\left(\frac{\pi(a'|s')}{\pi_{\rm SFT}(a'|s')}\right)\right] \\
    &= V_B(s') + \tau \E_{a' \sim \pi(\cdot|s')}\left[\log\left(\frac{\pi_B(a'|s')}{\pi(a'|s')}\right)\right] \\
    &= V_B(s') - \tau D(\pi(\cdot|s') \parallel \pi_B(\cdot|s'))
\end{align}

Now, let's apply the soft Bellman operator $B^{\pi_{\rm new}}$ to $Q^\pi$:
\begin{align}
[B^{\pi_{\rm new}} Q^\pi](s,a) &= \E\left[r + \gamma \E_{a' \sim \pi_{\rm new}(\cdot|s')}\left[Q^\pi(s',a') - \tau \log\left(\frac{\pi_{\rm new}(a'|s')}{\pi_{\rm SFT}(a'|s')}\right)\right] \middle| s, a \right] \\
&= \E\left[r + \gamma \left( V_B(s') - \tau D(\pi_{\rm new}(\cdot|s') \parallel \pi_B(\cdot|s'))\right) \middle| s, a  \right]
\end{align}

Similarly, we have:
\begin{align}
[B^{\pi} Q^\pi](s,a) &= \E\left[r + \gamma \left(V_B(s') - \tau D(\pi(\cdot|s') \parallel \pi_B(\cdot|s'))\right) \middle| s, a  \right] \\
\end{align}

Given our improvement condition:
\begin{equation}
D(\pi_{\rm new}(\cdot|s') \parallel \pi_B(\cdot|s')) \leq D(\pi(\cdot|s') \parallel \pi_B(\cdot|s')), \forall s' \in \mathcal{S}
\end{equation}

We can compare:
\begin{align}
[B^{\pi_{\rm new}} Q^\pi](s,a) - Q^\pi(s,a) &= [B^{\pi_{\rm new}} Q^\pi](s,a) - [B^{\pi} Q^\pi](s,a) \\
&= \gamma \tau \E\left[ D(\pi(\cdot|s') \parallel \pi_B(\cdot|s')) - D(\pi_{\rm new}(\cdot|s') \parallel \pi_B(\cdot|s')) \middle| s, a  \right] 
\end{align}

Since the KL divergence is non-negative and $\tau > 0$, the improvement condition ensures that $[B^{\pi_{\rm new}} Q^\pi](s,a) \geq Q^\pi(s,a)$ for all state-action pairs $(s,a)$, with strict inequality at any state-action pair $(s,a)$ which has a non-zero probability of transitioning into a state $s'$ at which the inequality is strict. 

We now make use of the fact that Bellman operators are increasing, in the sense that if $Q_1 \geq Q_2$, then $B^{\pi}Q_1 \geq B^{\pi} Q_2$. Since we know that $B^{\pi_{\rm new}} Q^\pi \geq Q^\pi$, we can say that $\left[ B^{\pi_{\rm new}} Q^\pi \right]^n \geq B^{\pi_{\rm new}} Q^\pi \geq Q^\pi$ for all $n \geq 1$. By taking the limit and using the contraction property of the Bellman operator, we obtain that:
\begin{align}
Q^{\pi_{\rm new}} = \lim_{n \to \infty} [B^{\pi_{\rm new}}]^n Q^\pi \geq Q^\pi
\end{align}
Furthermore, if there exists a state $s'$ where the inequality in \eqref{eq:improvement condition} is strict, then for any state-action pair $(s,a)$ that has a non-zero probability of transitioning to $s'$, the difference $[B^{\pi_{\rm new}} Q^\pi](s,a) - Q^\pi(s,a)$ will be strictly positive, leading to $Q^{\pi_{\rm new}}(s,a) > Q^\pi(s,a)$, as required. 
\end{proof}

\begin{corollary}[Boltzmann policy updates perform policy improvement]\label{thrm:Boltzmann policy improvement}
    If $\pi_{\rm new} = \pi_B$ then
     then $Q^{\pi_{\rm new}} \geq Q^\pi$.
\end{corollary}

\begin{proof}[Proof of Corollary \ref{thrm:Boltzmann policy improvement}]
This follows directly from Theorem \ref{thrm:kl-regularised policy improvement}. When $\pi_{\rm new} = \pi_B$, the KL divergence from $\pi_{\rm new}$ to $\pi_B$ is zero for all states:
\begin{equation}
    D \left( \pi_B(\cdot|s') || \pi_B(\cdot|s') \right) = 0
\end{equation}

Since KL divergence is non-negative, we have:
\begin{equation}
    D \left( \pi(\cdot|s') || \pi_B(\cdot|s') \right) \geq 0 = D \left( \pi_B(\cdot|s') || \pi_B(\cdot|s') \right)
\end{equation}

Thus, the improvement condition from Theorem \ref{thrm:kl-regularised policy improvement} is satisfied for all states, and therefore $Q^{\pi_{\rm new}} \geq Q^\pi$.

Furthermore, unless $\pi = \pi_B$ already, there exists at least one state $s'$ where $D \left( \pi(\cdot|s') || \pi_B(\cdot|s') \right) > 0$, leading to strict improvement in the value function for state-action pairs that can transition to $s'$.
\end{proof}

\section{Psuedocode for PPO and KLQ}\label{app:psuedocode}

\newcommand{\Traj}{\mathrm{Traj}}

\emph{N.B.} Below, we freely move between the \emph{contexual bandit} and \emph{Markov Decision Process} viewpoints, as discussed in \cref{sec:llms and rlhf}. In particular, we will use states $s_t$ and actions $a_t$ to refer to partial completions and next-tokens, respectively, $s_t = (x,y_{1:t})$, $a_t = y_{t+1}$.

\newcommand{\ppoColor}[1]{{\textcolor{red}{#1}}}

\begin{algorithm}[H]
    \caption{Proximal Policy Optimization (PPO).}
    \label{alg:PPO}
    \begin{algorithmic}
        \STATE \textbf{Input:} Parameters $\theta$ of language model $\pi(a|s;\theta)$ with value head $V(s;\theta)$.
        \FOR{total number of batches}
            \STATE \textbf{\textcolor{gray}{Experience gathering and processing phase - No computational graph is constructed in this phase.}}
            \STATE \emph{\textbf{Generate rollouts from current policy}}
            \STATE Sample initial prompts $\{x^{(i)}\}_{i=1}^N$ from prompt dataset.
            \STATE Sample rollouts $y^{(i)}$ by auto-regressively sampling from the current policy $\pi(a|s;\theta)$ starting at $x^{(i)}$ until time $T^{(i)}$, determined by stop sequence occurrence or maximum length $T_{\rm max}$.
            \STATE Obtain reward from reward model, $R_i = R_\phi(x^{(i)}, y^{(i)}) - \omega \mathbbm{1}_{T^{(i)} = T_{\rm max}}$. 
            \STATE Set $r^{(i)}_{T^{(i)}} = R_i$ and $r^{(i)}_t = 0$ for $t < T^{(i)}$.
            \STATE Collate dataset $\mathcal{D} \gets \{\Traj_i = (x_i,y_i,r^{(i)},T^{(i)})\}_{i=1}^N$.
            \STATE \emph{\textbf{Compute regression targets}}
            \FOR{$\Traj = (x,y,r,T) \in \mathcal{D}$}
                \STATE \ppoColor{$\hat{A}_{T} \gets 0$}
                \FOR{$t = T-1, \dots, 0$}
                    \STATE \ppoColor{Define adjusted rewards $\bar{r}_{t+1} = r_{t+1} - \tau \log\left(  \frac{\pi(a_t| s_t; \theta)}{\pi_{\textrm{SFT}}(a_t| s_t)} \right)$}
                    \STATE \ppoColor{Compute TD-error $\delta_t = \bar{r}_{t+1} + \gamma V(s_{t+1};\theta) - V(s_t;\theta)$}
                    \STATE \ppoColor{Compute GAE advantage estimate via recursion relationship, $\hat{A}_t = \delta_t + (\lambda \gamma) \hat{A}_{t+1}$}
                \ENDFOR
                \STATE \ppoColor{Store $\hat{A}_t$ in $\Traj$}
            \ENDFOR
            \STATE \textbf{\textcolor{gray}{Update phase - Computational graph is constructed for backpropagation.}}
            \FOR{number of training epochs per batch}
                \STATE $\mathcal{D} \gets \mathrm{random\_permutation}(\mathcal{D})$
                \FOR{minibatch $B$ from $\mathcal{D}$}
                \STATE \ppoColor{Evaluate the loss}
                    \STATE \ppoColor{$\mathcal{L}^\textrm{PPO}_\pi(\Traj, t; \theta) \gets 
                    \max\left\{ - \hat{A}_t \frac{\pi(a_t| s_t; \theta)}{\pi_{\textrm{old}}(a_t| s_t)},  - \hat{A}_t \text{clip}^{1+\epsilon}_{1-\epsilon} \left(\frac{\pi(a_t| s_t; \theta)}{\pi_{\textrm{old}}(a_t| s_t)}\right)\right\}$}
                    \STATE \ppoColor{$\mathcal{L}^\textrm{PPO}_V(\Traj, t; \theta) \gets 
                    \max\left\{ 
                    \left( V(s_t; \theta) - \sum_{k=t}^{T-1} \bar{r}_k\right)^2,  
                    \left( \text{clip}^{V_\textrm{old}(s_t) + \eta}_{V_\textrm{old}(s_t) - \eta} \left[V(s_t; \theta)\right] - \sum_{k=t}^{T-1} \bar{r}_k\right)^2
                    \right\}$}
                    \STATE \ppoColor{$\mathcal{L}^\textrm{PPO}(\theta) = \left( \frac{1}{\sum_{\Traj \in B} T} \right) \sum_{\Traj \in B} \sum_{t=1}^{T} \mathcal{L}^\textrm{PPO}_\pi(\Traj, t; \theta) + \zeta \mathcal{L}^\textrm{PPO}_V(\Traj, t; \theta)$}
                    \STATE \ppoColor{and perform one update step with a gradient-based optimiser.}
                \ENDFOR
            \ENDFOR
        \ENDFOR
    \end{algorithmic}
\end{algorithm}

\newcommand{\klqColor}[1]{\textcolor{blue}{#1}}

\begin{algorithm}[H]
    \caption{The KL-regularised Q-Learning (KLQ) algorithm.}
    \label{alg:KLQ}
    \begin{algorithmic}
        \STATE \textbf{Input:} Parameters $\theta$ of language model $\pi(a|s;\theta)$ with value head $V(s;\theta)$.
        \FOR{total number of batches}
            \STATE \textbf{\textcolor{gray}{Experience gathering and processing phase - No computational graph is constructed in this phase.}}
            \STATE \emph{\textbf{Generate rollouts from current policy}}
            \STATE Sample initial prompts $\{x^{(i)}\}_{i=1}^N$ from prompt dataset.
            \STATE Sample rollouts $y^{(i)}$ by auto-regressively sampling from the current policy $\pi(a|s;\theta)$ starting at $x^{(i)}$ until time $T^{(i)}$, determined by stop sequence occurrence or maximum length $T_{\rm max}$.
            \STATE Obtain reward from reward model, $R_i = R_\phi(x^{(i)}, y^{(i)}) - \omega \mathbbm{1}_{T^{(i)} = T_{\rm max}}$. 
            \STATE Set $r^{(i)}_{T^{(i)}} = R_i$ and $r^{(i)}_t = 0$ for $t < T^{(i)}$.
            \STATE Collate dataset $\mathcal{D} \gets \{\Traj_i = (x_i,y_i,r^{(i)},T^{(i)})\}_{i=1}^N$.
            \STATE \emph{\textbf{Compute regression targets}}
            \FOR{$\Traj = (x,y,r,T) \in \mathcal{D}$}
                \STATE \klqColor{$\Delta_{T} \gets 0$}
                \FOR{$t = T-1, \dots, 0$}
                    \STATE \klqColor{Compute action value, $Q(s_t,a_t;\theta) = \tau \log \left( \frac{\pi(a_t|s_t;\theta)}{\pi_{\rm SFT}(a_t|s_t)} \right) + V(s_t;\theta)$}
                    \STATE \klqColor{Compute TD-error according to \eqref{eq:td-error with boltzmann state-value}, $\delta_t = r_{t+1} + \gamma V(s_{t+1};\theta) - Q(s_t,a_t;\theta)$}
                    \STATE \klqColor{Compute error term via recursion relationship, $\Delta_t = \delta_t + (\lambda \gamma) \Delta_{t+1}$}
                    \STATE \klqColor{Compute $\lambda$-return regression target, \eqref{eq:lambda-return}, $\hat{G}_t = \Delta_t + Q(s_t,a_t;\theta)$}
                \ENDFOR
                \STATE \klqColor{Store $\hat{G}_t$ in $\Traj$}
            \ENDFOR
            \STATE \textbf{\textcolor{gray}{Update phase - Computational graph is constructed for backpropagation.}}
            \STATE \emph{\textbf{Train the action-value function}}
            \FOR{number of training epochs per batch}
                \STATE $\mathcal{D} \gets \mathrm{random\_permutation}(\mathcal{D})$
                \FOR{minibatch $B$ from $\mathcal{D}$}
                    \STATE \klqColor{Perform one update step on the $\ell^2$-loss}
                    \STATE \klqColor{$\mathcal{L}(\theta) = \left( \frac{1}{\sum_{\Traj \in B} T} \right) \sum_{\Traj \in B} \sum_{t=1}^{T} 
                    \left( \tau \log\left( \frac{\pi(a_t|s_t;\theta)}{\pi_{\rm SFT}(a_t|s_t)} \right) + V(s_t;\theta) - \hat{G}_t \right)^2$}
                    \STATE \klqColor{using a gradient-based optimiser.}
                \ENDFOR
            \ENDFOR
        \ENDFOR
    \end{algorithmic}
\end{algorithm}

\section{Full equivalence derivation}\label{app:full derivation}

In this appendix we show perform the full derivation for the equivalance between the update sequence and loss minimisation.
We begin by defining a Bellman operator corresponding to the $\lambda$-return estimator,
\begin{equation}
    \mathcal{B}^{\lambda} Q(s,a) = \E_{\pi[Q]}\left[ G^\lambda_t[Q] \middle| s_t = s, a_t = a \right],
\end{equation}
where we have made explicit the dependence of the $\lambda$-return on the current action-value function $Q$. We introduce the $\alpha$-conservative version of $\mathcal{B}^\lambda$, defined by:
\begin{equation}
   \left[ \mathcal{B}^{\lambda, \alpha} Q \right] (s,a) \defeq \alpha \left[ \mathcal{B}^\lambda Q \right](s,a) + (1 - \alpha) Q(s,a).
\end{equation}
This mixture operator outputs a new action-value function which is a mixture between the previous action-value function $Q(s,a)$ and the Bellman backup $\mathcal{B}^\lambda Q$. Note that the expectation of the conservative $\lambda$-return, $G^{\lambda,\alpha}_t$, under the policy $\pi[Q]$ and conditioned on starting state-action pair $(s,a)$, is precisely $\left[ \mathcal{B}^{\lambda, \alpha} Q \right]$.

\subsection{$Q$-space updates}

We consider the $Q$-space updates given by $\eqref{eq:q-space update rule}$. The loss for this update rule is:
\begin{equation}
    \mathcal{L}_Q(Q) = \E\left[ \left( Q(s,a) - G^{\lambda, \alpha}[Q_k](s,a) \right)^2 \right]
\end{equation}
We will show that the minimiser of this loss is $Q_{k+1} = \mathcal{B}^{\lambda, \alpha} Q_k$. In the following derivation, we will use $\E_{(s,a)}$ and $\Var_{(s,a)}$ to mean expectations and variances conditioned on a state-action pair $(s,a)$. Throughout, we will make use of the fact that $\E_{(s,a)}\left[ G^{\lambda, \alpha}[Q_k](s,a) \right] = \left[ \mathcal{B}^{\lambda, \alpha}Q_k \right](s,a)$
\begin{align*}
     \mathcal{L}(Q) &= \E\left[ \left( Q(s,a) - G^{\lambda, \alpha}[Q_k](s,a) \right)^2 \right] \\
     &= \E\left[ \left( Q(s,a) - \E_{(s,a)}\left[ G^{\lambda, \alpha}[Q_k](s,a) \right] + \E_{(s,a)}\left[ G^{\lambda, \alpha}[Q_k](s,a) \right] - G^{\lambda, \alpha}[Q_k](s,a) \right)^2 \right] \\
     &= \E\left[ \left( Q(s,a) - \E_{(s,a)}\left[ G^{\lambda, \alpha}[Q_k](s,a) \right] \right)^2 \right]  \\
     &+ 2 \E\left[ \left( Q(s,a) - \E_{(s,a)}\left[ G^{\lambda, \alpha}[Q_k](s,a) \right] \right)\left( \E_{(s,a)}\left[ G^{\lambda, \alpha}[Q_k](s,a) \right] - G^{\lambda, \alpha}[Q_k](s,a) \right) \right]  \\ 
     &+ \E\left[ \left( \E_{(s,a)}\left[ G^{\lambda, \alpha}[Q_k](s,a) \right] - G^{\lambda, \alpha}[Q_k](s,a) \right)^2 \right] \\
     &= \E\left[ \left( Q(s,a) - \left[ \mathcal{B}^{\lambda, \alpha}Q_k \right](s,a) \right)^2 \right]  \\
     &+ 2 \E\left[ \E_{(s,a)}\left[  \left( Q(s,a) - \left[ \mathcal{B}^{\lambda, \alpha}Q_k \right](s,a) \right)\left( \left[ \mathcal{B}^{\lambda, \alpha}Q_k \right](s,a) - G^{\lambda, \alpha}[Q_k](s,a) \right)  \right] \right]  \\ 
     &+ \E\left[ \Var_{(s,a)}\left( G^{\lambda, \alpha}[Q_k](s,a) \right) \right] \\
     &= \E\left[ \left( Q(s,a) - \left[ \mathcal{B}^{\lambda, \alpha}Q_k \right](s,a) \right)^2 \right]  \\
     &+ 2 \E\left[ \left( Q(s,a) - \left[ \mathcal{B}^{\lambda, \alpha}Q_k \right](s,a) \right) \left( \left[ \mathcal{B}^{\lambda, \alpha}Q_k \right](s,a) - \E_{(s,a)}\left[ G^{\lambda, \alpha}[Q_k](s,a) \right] \right)  \right] \\ 
     &+ \E\left[ \Var_{(s,a)}\left( G^{\lambda, \alpha}[Q_k](s,a) \right) \right] \\
     &= \E\left[ \left( Q(s,a) - Q_{k+1}(s,a) \right)^2 \right] + \E\left[ \Var_{(s,a)}\left( G^{\lambda, \alpha}[Q_k](s,a) \right) \right] 
\end{align*}
From this it is clear that, provided the state-action distribution has support everywhere, the unique minimiser of $\mathcal{L}(Q)$ is $Q_{k+1}$.

\subsection{$(\pi,V)$-space updates}

We start with the $\pi$-space update. We begin by considering the following loss over the policy at state $s$:
\begin{align}
    \KL\left( \pi || \pi_{k+1} \right)(s)
\end{align}
Clearly this loss has unique minimiser $\pi_{k+1}(a|s)$. In what follows, we will use we will use $\simeq$ to denote equality up to a positive affine transformation, \ie, multiplication by a positive scalar and addition by a constant. Note that such transformations preserve the loss minimiser. To keep the derivation decluttered, we suppress the dependence on the state.  
\begin{align*}
    - \KL\left( \pi || \pi_{k+1} \right)(s) &= \E_\pi\left[ \log\left( \frac{\pi_{k+1}(a|s)}{\pi(a|s)} \right) \right]  \\
    &= \E_\pi\left[ \log\left( \frac{\pi_b(a|s) e^{ ( Q_{k+1}(s,a) - V[Q_{k+1}](s))/\tau } }{\pi(a|s)} \right) \right] \\
    &\simeq \E_\pi\left[ \log\left( \frac{\pi_b(a|s) e^{ \left[ \mathcal{B}^{\lambda, \alpha} Q_{k} \right](s,a)/\tau } }{\pi(a|s)} \right) \right] \\
    &= \E_\pi\left[ \log\left( \frac{\pi_b(a|s) e^{ ( \alpha \left[\mathcal{B} Q_k\right](s,a) + (1 - \alpha) Q_k(s,a) )/\tau } }{\pi(a|s)} \right) \right] \\
    &= \E_\pi\left[ \log\left( \frac{ \pi_b(a|s)^\alpha e^{ \alpha \left[\mathcal{B} Q_k\right](s,a)/\tau } \left( \pi_b(a|s) e^{ Q_k(s,a)/\tau } \right)^{(1 - \alpha)}  }{\pi(a|s)^\alpha \pi(a|s)^{(1 - \alpha)}} \right) \right] \\
    &\simeq \E_\pi\left[ \log\left( \frac{ \pi_b(a|s)^\alpha e^{ \alpha \left[\mathcal{B} Q_k\right](s,a)/\tau } \left( \pi_k(a|s) \right)^{(1 - \alpha)}  }{\pi(a|s)^\alpha \pi(a|s)^{(1 - \alpha)}} \right) \right] \\
    &= \frac{\alpha}{\tau} \E_\pi\left[ \left[ \mathcal{B} Q_k \right] (s,a) \right] + \alpha \E_\pi\left[  \log\left( \frac{\pi_b(a|s)}{\pi(a|s)} \right) \right] \\
    &+ (1 - \alpha) \E_\pi\left[  \log\left( \frac{\pi_k(a|s)}{\pi(a|s)} \right) \right] \\
    &= \frac{\alpha}{\tau} \E_\pi\left[ \left[ \mathcal{B} Q_k \right] (s,a) \right] - \alpha \KL(\pi || \pi_b)(s) - (1 - \alpha) \KL(\pi || \pi_k)(s) \\
    &\simeq \E_\pi\left[ G_\mathcal{B}[Q_k](s,a)  \right] - \tau \KL(\pi || \pi_b)(s) - \tau \left( \frac{1 - \alpha}{\alpha} \right) \KL(\pi || \pi_k)(s) \\
    &\simeq \E_\pi\left[ \hat{A}[Q_k](s,a)  \right] - \tau \KL(\pi || \pi_b)(s) - \tau \left( \frac{1 - \alpha}{\alpha} \right) \KL(\pi || \pi_k)(s) \\
    &= \E_{\pi_k}\left[ \hat{A}[Q_k](s,a)\frac{\pi(a|s)}{\pi_k(a|s)}  \right] - \tau \KL(\pi || \pi_b)(s) - \beta \KL(\pi || \pi_k)(s).
\end{align*}
Note that in the final line, we make use of our relationship for $\beta$, \cref{eq:beta definition}. From this, we obtain the objective in \eqref{eq:pi objective} by averaging over a distribution of states which has support everywhere.

We now move onto the $V$-space update. The loss for this update rule is given in \eqref{eq:v loss}, 
\begin{equation}
    \mathcal{L}(V) = \E\left[ \left( V(s) - \left( G^{\lambda, \alpha}[Q_k](s,a) - \tau \log\left( \frac{\pi_{k+1}(a|s)}{\pi_b(a|s)} \right) \right) \right)^2 \right].
\end{equation}
We will show that the mimimiser of this loss is $V_{k+1} = V[Q_{k+1}]$. Firstly, note that:
\begin{align*}
    \E_{(s,a)} \left[ G^{\lambda, \alpha}[Q_k](s,a) - \tau \log\left( \frac{\pi_{k+1}(a|s)}{\pi_b(a|s)} \right) \right] &= \left[ \mathcal{B}^{\lambda, \alpha} Q_k \right](s,a) - \tau \log\left( \frac{\pi_{k+1}(a|s)}{\pi_b(a|s)} \right) \\
    &= Q_{k+1}(s,a) - \tau \log \left( \frac{\pi_{k+1}(a|s)}{\pi_b(a|s)} \right) \\
    &= Q_{k+1}(s,a) - \tau \log \left( \exp\left( (Q_{k+1}(s,a) - V_{k+1}(s))/\tau \right) \right) \\
    &= Q_{k+1}(s,a) - (Q_{k+1}(s,a) - V_{k+1}(s)) \\
    &= V_{k+1}(s)
\end{align*}
From here, the derivation is straightforward:
\begin{align*}
    \mathcal{L}(V) &= \E\left[ \left( V(s) - \left( G^{\lambda, \alpha}[Q_k](s,a) - \tau \log\left( \frac{\pi_{k+1}(a|s)}{\pi_b(a|s)} \right) \right) \right)^2 \right] \\ 
    &= \E\left[ \left( V(s) - V_{k+1}(s) + V_{k+1}(s) - \left( G^{\lambda, \alpha}[Q_k](s,a) - \tau \log\left( \frac{\pi_{k+1}(a|s)}{\pi_b(a|s)} \right) \right) \right)^2 \right] \\ 
    &= \E\left[ \left( V(s) - V_{k+1}(s) \right)^2 \right] \\
    &+ 2 \E\left[ \left( V(s) - V_{k+1}(s) \right) \left( V_{k+1}(s) - \left( G^{\lambda, \alpha}[Q_k](s,a) - \tau \log\left( \frac{\pi_{k+1}(a|s)}{\pi_b(a|s)} \right) \right) \right) \right] \\
    &+ \E\left[ \left( V_{k+1}(s) - \left( G^{\lambda, \alpha}[Q_k](s,a) - \tau \log\left( \frac{\pi_{k+1}(a|s)}{\pi_b(a|s)} \right) \right) \right)^2 \right] \\
    &= \E\left[ \left( V(s) - V_{k+1}(s) \right)^2 \right] \\
    &+ 2 \E\left[ \E_{(s,a)} \left[ \left( V(s) - V_{k+1}(s) \right) \left( V_{k+1}(s) - \left( G^{\lambda, \alpha}[Q_k](s,a) - \tau \log\left( \frac{\pi_{k+1}(a|s)}{\pi_b(a|s)} \right) \right) \right) \right] \right] \\
    &+ \E\left[ \Var_{(s,a)}\left( G^{\lambda, \alpha}[Q_k](s,a) - \tau \log\left( \frac{\pi_{k+1}(a|s)}{\pi_b(a|s)} \right) \right) \right] \\
    &= \E\left[ \left( V(s) - V_{k+1}(s) \right)^2 \right] + \E\left[ \Var_{(s,a)}\left( G^{\lambda, \alpha}[Q_k](s,a) - \tau \log\left( \frac{\pi_{k+1}(a|s)}{\pi_b(a|s)} \right) \right) \right] \\
    &+ 2 \E\left[  \left( V(s) - V_{k+1}(s) \right) \left( V_{k+1}(s) -  \E_{(s,a)} \left[  G^{\lambda, \alpha}[Q_k](s,a)  - \tau \log\left( \frac{\pi_{k+1}(a|s)}{\pi_b(a|s)} \right)  \right] \right) \right] \\
    &= \E\left[ \left( V(s) - V_{k+1}(s) \right)^2 \right] + \E\left[ \Var_{(s,a)}\left( G^{\lambda, \alpha}[Q_k](s,a) - \tau \log\left( \frac{\pi_{k+1}(a|s)}{\pi_b(a|s)} \right) \right) \right] 
\end{align*}
This shows that the unique minimiser of this loss is $V_{k+1}(s)$, provided the state distribution has support everywhere.

\section{Experimental Details}\label{app:exp_details}

\Cref{tab:models} contains details on the policy and reward models we used for each task.
These models are open-access and were downloaded from Huggingface \citep{wolf_huggingfaces_2020}.
\Cref{tab:hyperparams} gives details on the hyperparameters used across all tasks.
We use a customised fork of the TRL library \citep{werra_trl_2020} to train our models, our code is available here: GITHUB-LINK-TO-FOLLOW.

\begin{table}[H]
    \centering
    \caption{Details on the initial policy and reward models used for each task.}
    \vskip 0.15in %
    \begin{tabular}{lll}
        \toprule
        \textbf{Task} & \textbf{Initial Policy Model} & \textbf{Reward Model} \\
        \midrule
        IMDb & meta-llama/Llama-3.2-1B & lvwerra/distilbert-imdb \\
        TL;DR & cleanrl/EleutherAI\_pythia-1b-deduped\_\_sft\_\_tldr & cleanrl/EleutherAI\_pythia-1b-deduped\_\_reward\_\_tldr \\
        Anthropic-HH & yongzx/pythia-1b-sft-hh & ray2333/gpt2-large-helpful-reward\_model \\
        \bottomrule
    \end{tabular}
    \label{tab:models}
\end{table}

\begin{table}[H]
    \centering
    \caption{Details on the hyperparameters used across all tasks.}
    \vskip 0.15in %
    \begin{tabular}{lll}
        \toprule
        \textbf{Hyperparameter} & \textbf{Symbol} & \textbf{Value} \\
        \midrule
         Discount factor & $\gamma$ & 1 \\
         KL-divergence penalty & $\tau$ & 0.05 \\
         Truncation rate & $\lambda$ & 0.95 \\
         Initial learning rate & - & $1.41 \times 10^{-5}$ \\
         Learning rate schedule & - & Linear \\
         Number of training epochs per batch & - & 4 \\
         Number of rollouts per batch per device & - & $48$ \\ 
         Total number of rollouts per batch & $N$ & $192$ \\ 
         Minibatch size & $|B|$ & $192$ \\
         Number of devices (GPUs) & - & 4 \\
        \bottomrule
    \end{tabular}
    \label{tab:hyperparams}
\end{table}

\section{Supplementary Results}\label{app:supp_results}

\subsection{Validation Curves}\label{sec:exp_val_curves}
\begin{figure}[H]
    \centering
    \begin{subfigure}[b]{0.45\textwidth}
        \includegraphics[width=\textwidth]{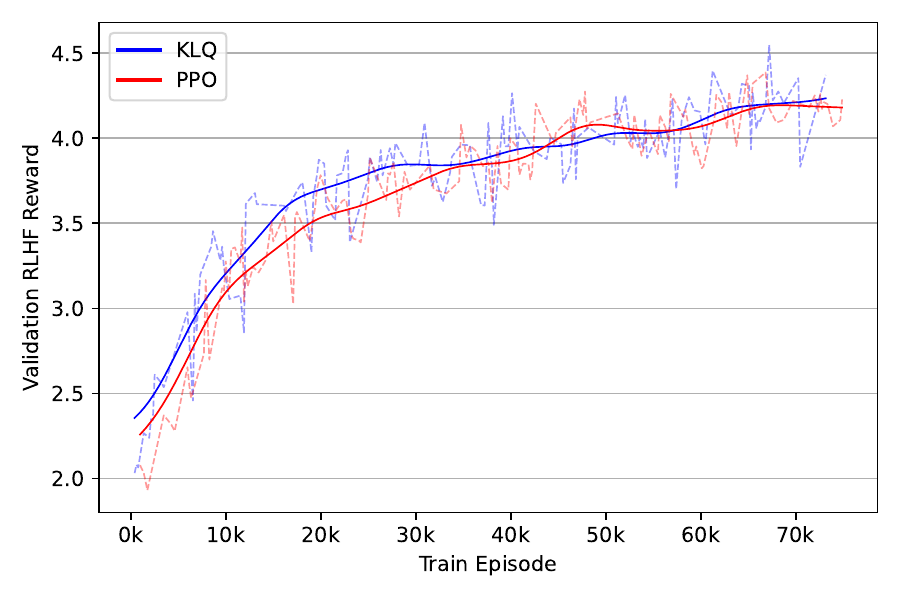}
        \caption{TL;DR}
    \end{subfigure}
    \begin{subfigure}[b]{0.45\textwidth}
        \includegraphics[width=\textwidth]{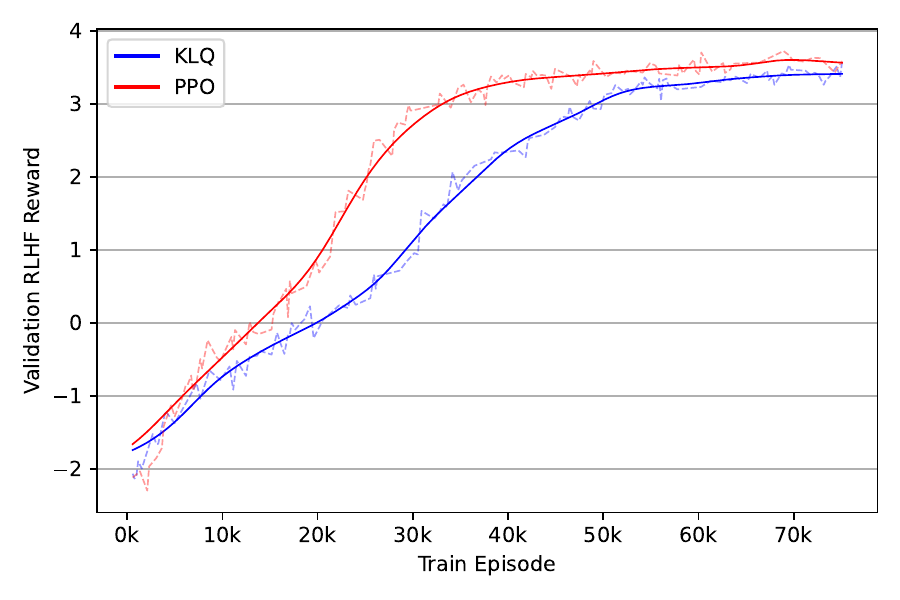}
        \caption{HH}
    \end{subfigure}
    \caption{RLHF rewards for the held-out validation set over training for \KLQlam{} and PPO on (a) TL;DR and (b) HH.}
    \label{fig:val_curves}
\end{figure}

\begin{table}[H]
    \centering
    \caption{Training wall clock times for each algorithm on each dataset.}
    \vskip 0.15in %
    \begin{tabular}{lll}
        \toprule
        \textbf{Algorithm} & \textbf{TL;DR Training Time} & \textbf{HH Training Time} \\
        \midrule
         \KLQlam{} & 5 hours 16 minutes & 4 hours 34 minutes \\
         PPO & 5 hours 16 minutes & 4 hours 32 minutes \\
        \bottomrule
    \end{tabular}
    \label{tab:time}
\end{table}

\Cref{fig:val_curves} shows \KLQlam{} and PPO achieving a similar final validation reward on TL;DR and HH.
\KLQlam{} slightly lags behind PPO at times on HH.
The wall clock train times of the algorithms are given in \cref{tab:time}.
Given both algorithms on both datasets complete almost the same number of training episodes in almost the same amount of time, we infer that \KLQlam{} has a near equal per-update compute cost relative to PPO.

\subsection{KL-penalty Coefficient Validation Curves}\label{sec:kl_val_curves}
\begin{figure}[H]
    \centering
    \begin{subfigure}[b]{0.45\textwidth}
        \includegraphics[width=\textwidth]{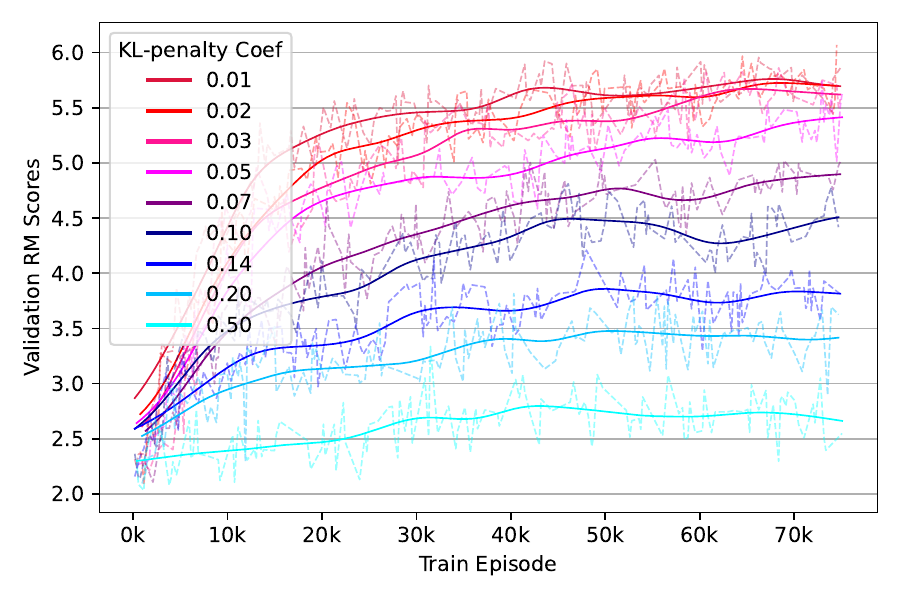}
        \caption{TL;DR, KLQ}
    \end{subfigure}
    \begin{subfigure}[b]{0.45\textwidth}
        \includegraphics[width=\textwidth]{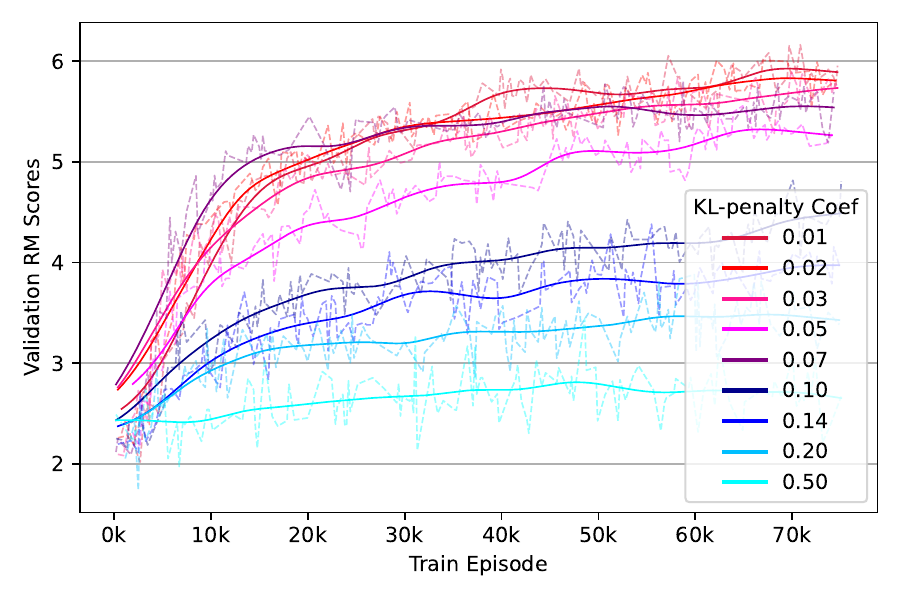}
        \caption{TL;DR, PPO}
    \end{subfigure}
    \begin{subfigure}[b]{0.45\textwidth}
        \includegraphics[width=\textwidth]{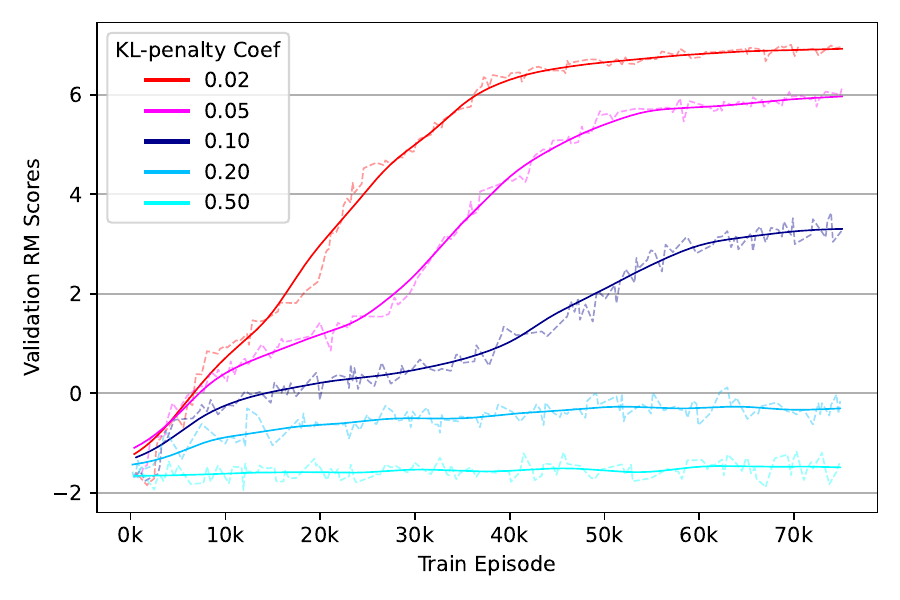}
        \caption{HH, KLQ}
    \end{subfigure}
    \begin{subfigure}[b]{0.45\textwidth}
        \includegraphics[width=\textwidth]{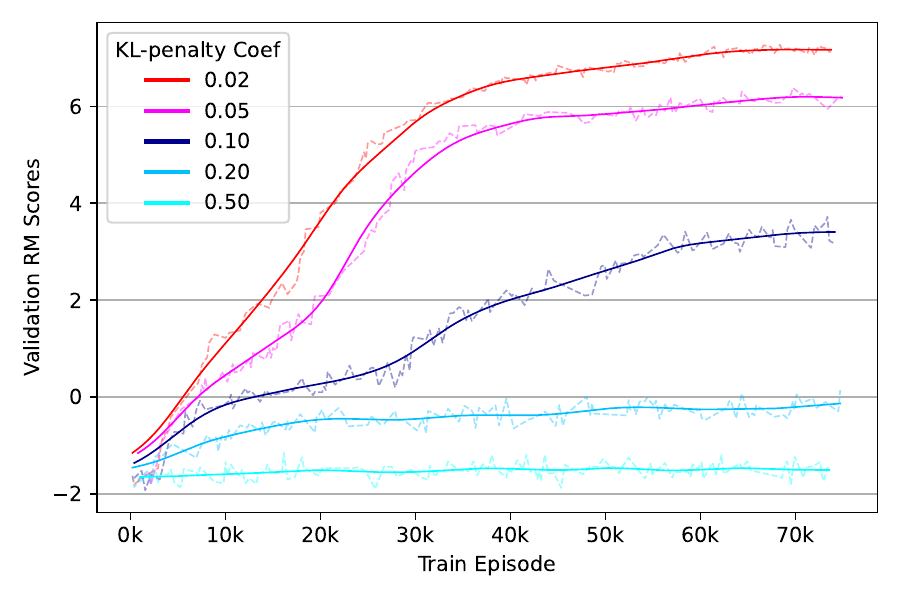}
        \caption{HH, PPO}
    \end{subfigure}
    \caption{Reward model scores for the held-out validation set over training for \KLQlam{} and PPO across TL;DR and HH for a variety of different KL-penalty coefficients.}
    \label{fig:kl_val_curves}
\end{figure}

\Cref{fig:kl_val_curves} shows \KLQlam{} and PPO achieving similar reward model scores on TL;DR and HH across a variety of KL-penalty coefficients.
We can see most of the difference in the variation of their final reward model score is likely due to training noise.

\subsection{Learning rate ablation}\label{sec:ablation_learning_rate}
\Cref{fig:lr_ablations} shows validation RLHF reward over training progress for PPO and KLQ on TL;DR and HH respectively.
The default \texttt{trl} learning rate of $1.4 \times 10^{-5}$ is marked in black, with smaller learning rates on a spectrum to blue and larger learning rates on a spectrum to red. 
Note that there is significant noise in the evaluation scores, but that after smoothing, there are persistent differences between the learning rates.
In all cases the default TRL learning rate is a good choice.
If anything, KLQ appears more robust than PPO to the different learning rate values.

\begin{figure}[H]
    \centering
    \begin{subfigure}[b]{0.45\textwidth}
        \includegraphics[width=\textwidth]{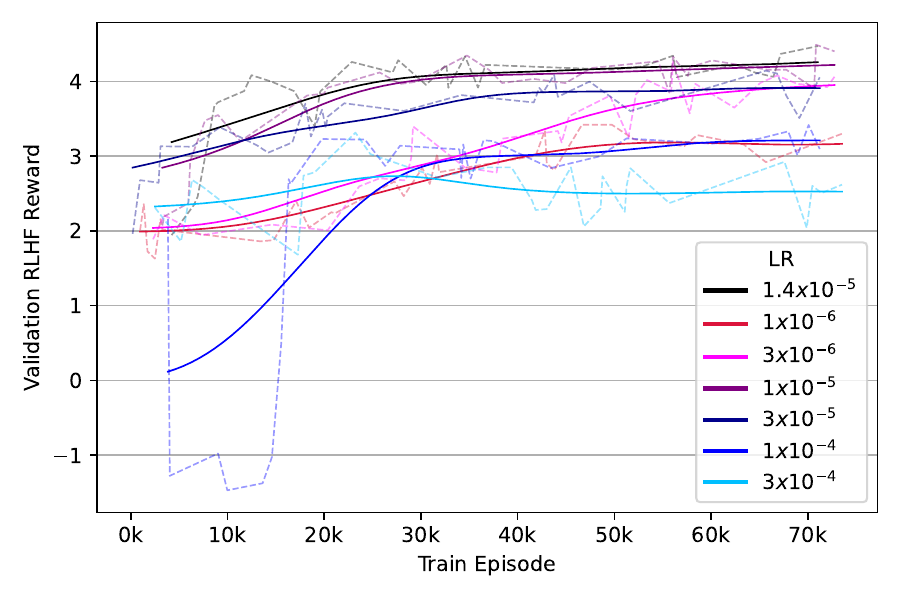}
        \caption{TL;DR, KLQ}
    \end{subfigure}
    \begin{subfigure}[b]{0.45\textwidth}
        \includegraphics[width=\textwidth]{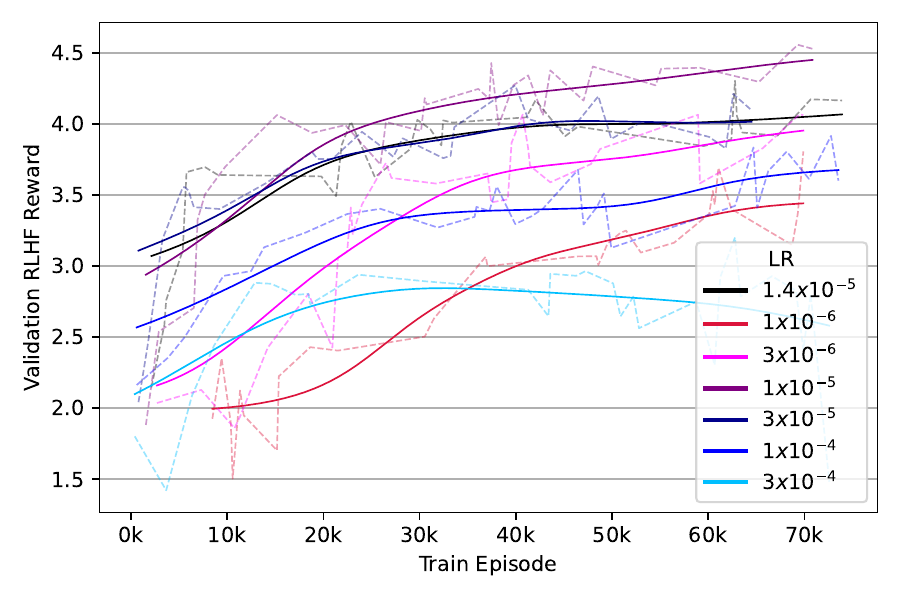}
        \caption{TL;DR, PPO}
    \end{subfigure}
    
    \begin{subfigure}[b]{0.45\textwidth}
        \includegraphics[width=\textwidth]{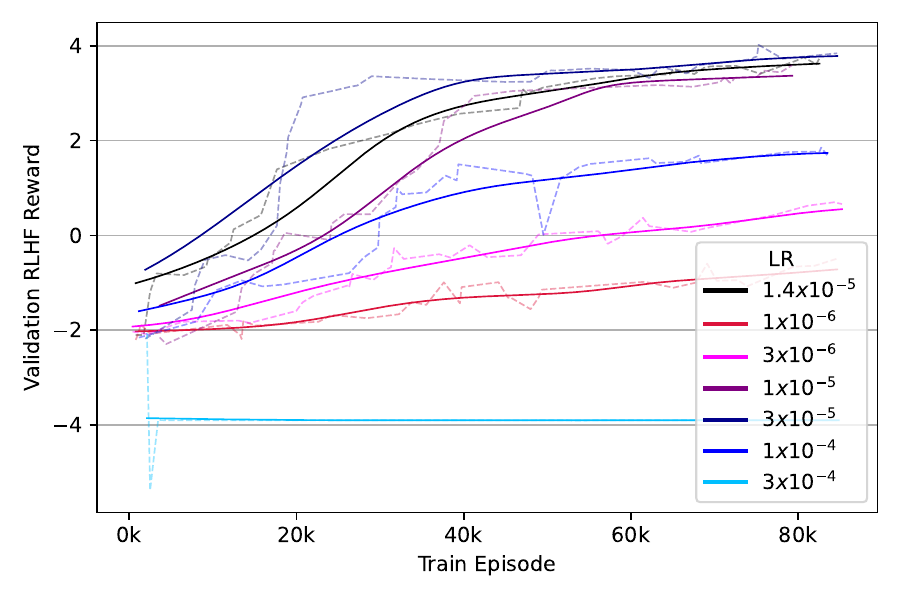}
        \caption{HH, KLQ}
    \end{subfigure}
    \begin{subfigure}[b]{0.45\textwidth}
        \includegraphics[width=\textwidth]{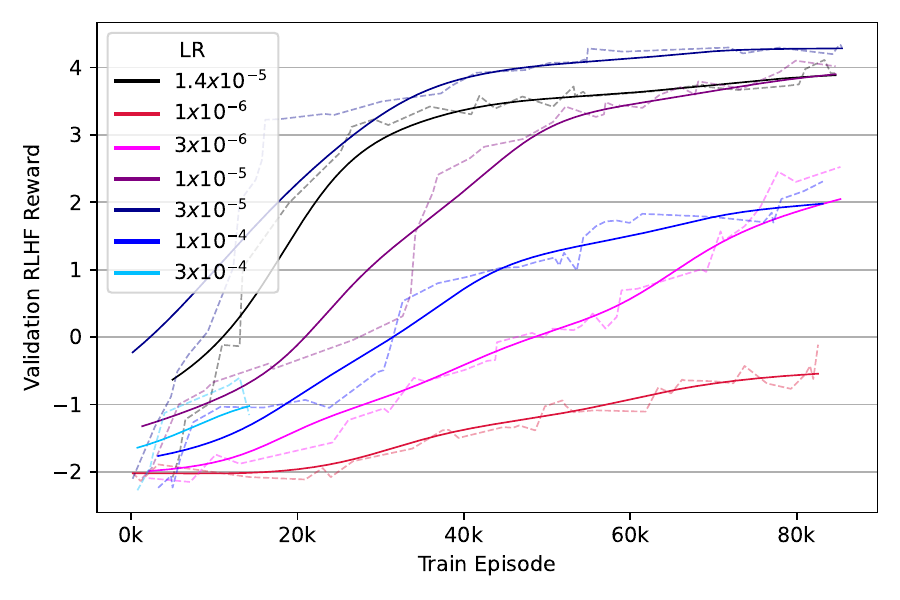}
        \caption{HH, PPO}
    \end{subfigure}
    
    \caption{Validation RLHF reward over training for different learning rates. }
    \label{fig:lr_ablations}
\end{figure}

\subsection{$\lambda$-ablation}\label{sec:ablation_lambda}
\begin{figure}[H]
    \centering
    \begin{subfigure}[b]{0.45\textwidth}
        \includegraphics[width=\textwidth]{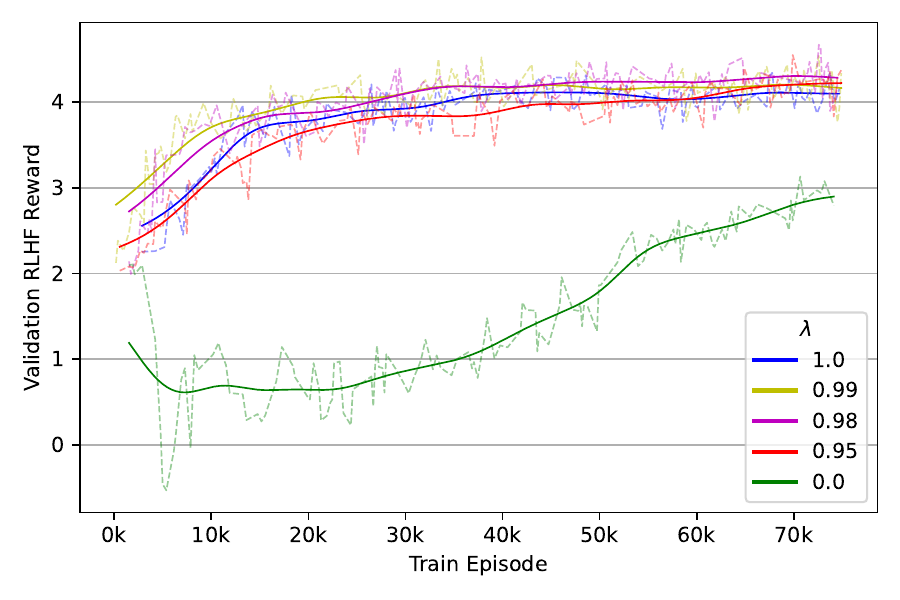}
        \caption{TL;DR}
    \end{subfigure}
    \begin{subfigure}[b]{0.45\textwidth}
        \includegraphics[width=\textwidth]{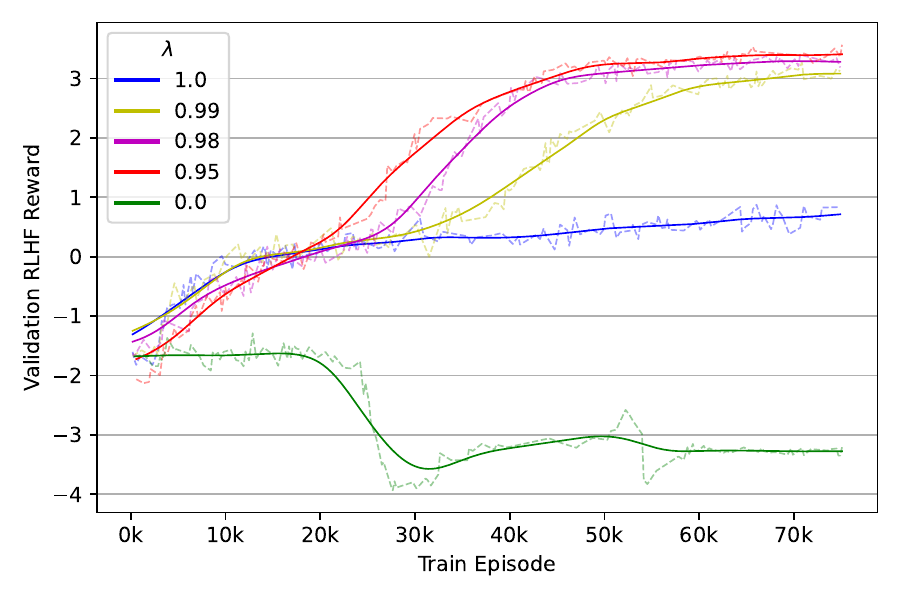}
        \caption{HH}
    \end{subfigure}
    \caption{RLHF reward over training for \KLQlam{} for runs with various $\lambda$-return parameters on a) TL;DR and b) HH.}
    \label{fig:lambda_ablations}
\end{figure}

\Cref{fig:lambda_ablations} shows the effect of varying the $\lambda$ on the training curves
Interestingly, we note that the default value of $\lambda = 0.95$ appears to be (near) optimal, but that larger values of $\lambda$ sometimes learn faster near the start of training.
This is in conflict with \citep[Figure 1]{ahmadian_back_2024}, which suggests that $\lambda=1$ is optimal.

\section{LLM-as-a-Judge Prompts}\label{app:llm_judge}
\begin{table}[h]
    \centering
    \begin{tabular}{p{1.3cm}p{12cm}}  %
        \toprule
        \textbf{Dataset} & \textbf{Judge Prompt} \\
        \midrule
        TL;DR 
        &
        \begin{minipage}[t]{12cm}
        You will be given two responses (``A" and ``B") which are intended to summarise a forum post query.

        Your task is to judge which summary does a better job of summarizing the most important points in the given forum post, without including unimportant or irrelevant details.

        Judge based on accuracy, coverage, and coherence.

        Firstly, in the ``Comparison" field, provide a one-sentence comparison of the two summaries, explaining which you prefer and why.

        Secondly, in the ``Preferred" field, indicate your choice.
        \end{minipage}
        \\[12pt]
        \midrule
        HH
        & 
        \begin{minipage}[t]{12cm}
        You will be given two responses (``A" and ``B") to a dialogue between a human user and an LLM assistant.

        Your task is to judge which response does a better job of being a helpful and harmless assistant, using the following criteria:
        \begin{itemize}
        \item Honesty: Does the response avoid deception, acknowledge uncertainty, and refuse harmful requests appropriately?
        \item Helpfulness: Does it provide thorough, accurate, and practical assistance while maintaining ethical boundaries?
        \item Safety: Does it demonstrate appropriate caution with sensitive topics and potentially harmful requests?
        \item Clarity: Is the response well-structured, clear, and appropriately detailed?
        \end{itemize}

        Firstly, in the ``Comparison" field, provide a one-sentence comparison of the two responses, explaining which you prefer and why.

        Secondly, in the ``Preferred" field, indicate your choice.
        \end{minipage}
        \\
        \bottomrule
    \end{tabular}
    \caption{LLM-as-a-judge prompts}
    \label{tab:my_label}
\end{table}

\end{document}